\documentclass[11pt,a4paper]{article}
\usepackage[utf8]{inputenc}
\usepackage[T1]{fontenc}
\usepackage{fullpage}
\usepackage{caption,subcaption,xspace,amsthm,hyperref,amsmath,amssymb,mathtools,algorithm,algpseudocode}
\newtheorem{theorem}{Theorem}
\newtheorem{lemma}{Lemma}
\newtheorem{remark}{Remark}
\newtheorem{corollary}{Corollary}
\newtheorem{definition}{Definition}

\usepackage[normalem]{ulem}
\usepackage{xcolor}
\usepackage{tikz}
\usetikzlibrary{positioning,matrix,fit,decorations.pathreplacing}

\tikzset{every node/.style={inner sep=0, outer sep=0},
    cmat/.style={matrix of nodes, execute at empty cell={\node[text=black!60!white] {0};}, left delimiter={[}, right delimiter={]}, minimum height=12pt, text width=12pt, text depth=0.5ex, align=center}}

\usepackage[backend=biber,style=numeric,sorting=none,maxbibnames=99]{biblatex}
\usepackage{booktabs}
\usepackage{color}

\usepackage[noblocks]{authblk}
\DeclareMathOperator*{\argmin}{arg\,min}

\newcommand{\E}{\mathbb{E}}
\newcommand{\Prob}{\mathbf{Pr}}

\newcommand{\MC}{\textsc{MC}\xspace}
\newcommand{\RH}{\textsc{RH}\xspace}
\newcommand{\HC}{\textsc{HC}\xspace}

\title{Robust Non-Clairvoyant Scheduling with Classification Models}

\author[1]{Anthony Dugois}
\author[1]{Vincent Fagnon}
\author[2]{Giorgio Lucarelli}

\affil[1]{Univ.\ Marie et Louis Pasteur, CNRS, institut FEMTO-ST, F-25000 Besançon, France}
\affil[ ]{{\small\texttt{\{anthony.dugois,vincent.fagnon\}@umlp.fr}}}
\affil[2]{Université de Lorraine, LCOMS, Metz, France}
\affil[ ]{{\small\texttt{giorgio.lucarelli@univ-lorraine.fr}}}

\date{}

\begin{document}

\maketitle

\begin{abstract}
We study the classical single-machine scheduling problem of minimizing the sum of completion times of jobs in a non-clairvoyant setting, where the processing time of each job remains unknown until its completion. This is a hard problem for which no constant competitive algorithm is possible. Inspired by robust optimization and learning-augmented algorithms, we introduce a novel robustness framework that leverages structural information provided by a classification model to overcome this limitation. Specifically, we assume that jobs are partitioned into classes and we have access to the confusion matrix of the classifier, whose entry $(k,\ell)$ indicates the number of jobs predicted to belong to class~$k$ but that actually belong to class~$\ell$. In this manner, we are able to characterize uncertainty as a set of permutations within each predicted class, rather than as a collection of discrete numerical scenarios, avoiding the computational difficulty of classical robust metrics, such as Min-Max and Min-Max Regret. In addition to these worst-case metrics, we also consider the expected objective over all scenarios. We first propose an optimal non-adaptive strategy that is oblivious with respect to all three robust criteria. We then investigate adaptive and randomized algorithms, showing that they can outperform the optimal non-adaptive strategy when the matrix exhibits particular structural properties.

\newcommand{\keywords}[1]{\par\bigskip\noindent\textbf{Keywords:} #1 \par\medskip}
\keywords{Scheduling, Robustness, Non-Clairvoyant, Classification, Confusion Matrix}
\end{abstract}

\newpage

\section{Introduction}

Scheduling problems, where one seeks to schedule a set of jobs on a set of resources in order to optimize a given objective function, are known to be particularly sensitive to uncertainty. The various applications of scheduling are often subject to different sources of variability, making the perfect knowledge of input parameters extremely challenging. Typically, it is common that the exact processing time of a task cannot be known with full precision before its completion. As a consequence, algorithms leveraging this information are more difficult to use in practice. Several approaches have been proposed to deal with this uncertainty, among which online/non-clairvoyant optimization~\cite{becchetti2004nonclairvoyant,kalyanasundaram2003minimizing,motwani1994nonclairvoyant}, stochastic optimization~\cite{im2015stochastic,megow2006models,mohring1999approximation,schulz2008stochastic,skutella2014stochastic}, robust optimization~\cite{kasperski2016single,mastrolilli2013single,shabtay2023stateoftheart}, or learning-augmented algorithms~\cite{azar2021flow,azar2022distortion,bampis2026non,im2023non,purohit2018improving}.

In the latter approach, one has access to an oracle that predicts the values of the unknown parameters. These predictions are subject to errors, typically summarized by a single scalar \(\eta\), which represents the ``overall quality'' of the oracle on a given instance. The goal is then to design an algorithm with a formal guarantee (e.g., a competitive ratio) that depends on \(\eta\). The oracle is often seen as a black box, in the sense that no assumptions are made about its behavior. However, in many practical applications, predictions may exhibit specific patterns that could be exploited by optimization algorithms. For example, an expert may specialize in recognizing a particular class of items, providing near-perfect predictions for the corresponding parameters while producing less accurate predictions for other types of items. Similarly, an oracle may be guaranteed never to overestimate (or underestimate) parameters values. Unfortunately, the single-scalar error representation used in the classical learning-augmented framework is generally too coarse to capture such structured prediction patterns in optimization problems.

In this paper, we introduce a novel robustness analysis framework based on \emph{classification models}, which assign objects to discrete categories. If we note \(X\subseteq\mathbb{R}^n\) the features describing the objects to classify and \(Y=\{c_1,c_2,\ldots,c_K\}\) a finite set of \(K\) class labels (representing the possible categories for our objects), a \emph{classifier} is a function \(\mathcal{C}:X\to Y\) mapping each vector \(x\in X\) to a class label \(y\in Y\). Such classifiers may be obtained through (un)supervised machine learning methods (e.g., Support Vector Machines, Decision Trees, Neural Networks, k-Nearest Neighbors, etc.), probabilistic data structures (e.g., Bloom Filters, Count-Min Sketch, etc.), or represent a human expert predicting the category of an object from simple observation. What makes classifiers interesting in our context is that they are typically evaluated through their \emph{confusion matrix}, which is obtained by running the model on the testing set. For a classifier \(\mathcal{C}\) on \(K\) classes, a confusion matrix is a \(K\times K\) matrix \(E_{\mathcal{C}}\) where each entry \(e_{ij}\) indicates the number of items of class \(c_j\) that \(\mathcal{C}\) predicted as \(c_i\). In other words, the rows of \(E_{\mathcal{C}}\) correspond to the predicted classes, while the columns correspond to the true classes of the items\footnote{Depending on the definition, the predictions may appear in the columns instead of the rows of the confusion matrix; everything is equivalent by taking the transpose of \(E_{\mathcal{C}}\).}. Fig.~\ref{fig.cm} illustrates the confusion matrix of a binary classifier.

\begin{figure}[htbp]
    \centering
    \begin{tikzpicture}
        \pgfdeclarelayer{bg}
        \pgfsetlayers{bg,main}

        \matrix[
            cmat,
            minimum height=18pt,
            text width=18pt
        ] (mat) {
            2 & 1 \\
            4 & 5 \\
        };

        \begin{pgfonlayer}{bg}
            \node[fit=(mat-1-1)(mat-2-1),fill=red,fill opacity=0.2] {};
            \node[fit=(mat-2-1)(mat-2-2),fill=blue,fill opacity=0.2] {};
        \end{pgfonlayer}

        \foreach\i in {1,...,2}{
            \node[above=6pt of mat-1-\i] (t-c\i) {\(c_\i\)};
            \node[left=10pt of mat-\i-1] (p-c\i) {\(c_\i\)};
        }

        \node[fit=(t-c1)(t-c2)] (true) {};
        \node[fit=(p-c1)(p-c2)] (pred) {};

        \node[above=6pt of true.north,font=\footnotesize\bfseries] {Actual};
        \node[left=6pt of pred.west,anchor=south,rotate=90,font=\footnotesize\bfseries] {Predicted};
    \end{tikzpicture}
    \caption{Confusion matrix of a binary classifier (2 classes). The entry on row \(i\) and column \(j\) indicates the number of items of class \(c_j\) predicted as \(c_i\). The entries highlighted in red (first column) represent the 6 items actually in class \(c_1\) (among which 2 are correctly classified and 4 are misclassified in \(c_2\)). Those highlighted in blue (second row) represent the 9 items classified as \(c_2\) (among which 4 are actually in \(c_1\)).}
    \label{fig.cm}
\end{figure}

Interestingly, it is possible to infer the performance of a classifier on unlabeled data~\cite{bialek2026estimating,kivimaki2026performance}, and in particular to estimate its confusion matrix. We therefore propose to leverage such confusion matrices in scheduling problems. Suppose we are given a set of jobs partitioned into \(K\) classes, where all jobs within the same class have identical processing times. The objective is to schedule these jobs \emph{non-preemptively} on a single machine so as to minimize the sum of completion times. In the clairvoyant setting, it is well-known that the Shortest Processing Time (SPT) policy is optimal~\cite{pinedo2012scheduling}. In the non-clairvoyant setting, the class of each job is unknown, and since preemption is not allowed, the best strategy is simply to execute the jobs in an arbitrary order~\cite{JagerSWW25}. The worst case outcome is then a Longest Processing Time (LPT) schedule.

Now suppose we have access to a classifier \(\mathcal{C}\) that predicts the class of each job, and that its confusion matrix on the problem instance is known in advance. For example, consider the two-class setting illustrated in Fig.~\ref{fig.cm}. Although we do not know whether any particular job is small or large, among the jobs that \(\mathcal{C}\) predicts to be large (the blue entries in Fig.~\ref{fig.cm}), we do know how many are actually small and how many are actually large. This additional information makes it possible to decide whether the classifier should be trusted, while always avoiding the worst-case LPT schedule. This naturally raises the question of whether the information captured by the confusion matrix can be exploited to guarantee the robustness of scheduling algorithms. The remainder of this paper is devoted to answering this question by developing a framework that explicitly leverages confusion matrices.

\section{Model Definition and Contributions}
\label{sec:model}

We revisit the famous \(1||\sum C_j\) scheduling problem. We are given a set \(\mathcal{J}=\{1,2,\ldots,n\}\) of \(n\) jobs, partitioned into \(K\) disjoint subsets \(\mathcal{J}_1,\mathcal{J}_2,\ldots,\mathcal{J}_K\). We say that jobs in \(\mathcal{J}_k\) belong to class \(k\). The \emph{processing time} of a job \(j\) is noted \(p_j\in\mathbb{Q}_{\ge 0}\). All jobs of a given class \(k\) have identical processing time \(\lambda_k\), and we assume that \(\lambda_1<\lambda_2<\cdots<\lambda_K\). The objective is to schedule these jobs \emph{non-preemptively} on a single machine in such a way that the sum \(\sum_{j\in\mathcal{J}}C_j\) is minimized, where \(C_j\) denotes the \emph{completion time} of job \(j\).

\medskip\noindent\textbf{Knowledge model.} We consider that we do not know the class of a job in advance: this information is revealed only at full completion of the job. However, we have access to a classifier \(\mathcal{C}\) that is able to \emph{predict} the class of a job before executing it. Each request to the classifier is free, and it is deterministic in the sense that it always gives the same answer for a given job. The classifier is not perfect and may respond with the wrong class. Although we do not have any information for each individual job, we consider that, for any instance of the problem, we know in advance the \emph{confusion matrix} \(E_{\mathcal{C}}\) of the classifier \(\mathcal{C}\), which consists in a \(K\times K\) integer matrix where each entry \(e_{k\ell}\) on row \(k\) and column \(\ell\) gives the number of jobs that the classifier believes to be in class \(k\) whereas they are in reality in class \(\ell\). For each \(k\), we note \(\mathcal{R}_k\) the set of jobs that the classifier believes to be in class \(k\); the number of jobs in \(\mathcal{R}_k\) is equal to the sum of entries in row \(k\) of the confusion matrix \(E_{\mathcal{C}}\). We also have \(e_{k\ell}=|\mathcal{R}_k\cap\mathcal{J}_\ell|\) for all \(k,\ell\). We assume that we systematically begin by executing the classifier on each job (as requesting the classifier is free) in order to build the sets \(\mathcal{R}_k\). Again, we do not know the class of each individual job, but we do know the distribution of each set \(\mathcal{R}_k\) thanks to the a priori knowledge of the confusion matrix. This knowledge may seem to be a strong assumption; however, recent approaches in the field of Machine Learning have demonstrated the possibility to estimate the performance of a trained classifier on unlabelled data, even in the case of distributional shift~\cite{bialek2026estimating,kivimaki2026performance}. In particular, it is possible to provide an estimation of confusion matrices. In this paper, we propose to evaluate the gain that a \emph{perfect} ``unsupervised performance estimation'' technique would permit on scheduling algorithms, and we leave the analysis of inaccurate confusion matrices for future work.

Interestingly, this knowledge model can be seen as a generalization of the classical clairvoyant and non-clairvoyant models. We say that a classifier \(\mathcal{C}\) is \emph{clairvoyant} if and only if each row of its confusion matrix \(E_{\mathcal{C}}\) has at most one nonzero entry. Indeed, since all jobs within a given class share the same processing time, knowing that all jobs in \(\mathcal{R}_k\) belong to a single true class \(\ell\) is sufficient to determine the processing time of each job prior to its execution. On the opposite side, any classifier \(\mathcal{C}\) whose confusion matrix \(E_{\mathcal{C}}\) has only one row with nonzero entries does not help at all, as it will predict that all jobs belong to the same class. Hence, we say that such a classifier is \emph{non-clairvoyant}. All the other classifiers are in between of these two extremes. Fig.~\ref{fig.cm-examples} gives several examples of matrices.

\begin{figure}[htbp]
    \centering
    \begin{subfigure}[t]{0.2\textwidth}
        \centering
        \begin{tikzpicture}
            \matrix[cmat] (mat) {
                2 &   &   \\
                  & 4 &   \\
                  &   & 3 \\
            };
        \end{tikzpicture}
        \caption{}
    \end{subfigure}\hfill
    \begin{subfigure}[t]{0.2\textwidth}
        \centering
        \begin{tikzpicture}
            \matrix[cmat] (mat) {
                  &   & 2 \\
                4 &   &   \\
                  & 3 &   \\
            };
        \end{tikzpicture}
        \caption{}
    \end{subfigure}\hfill
    \begin{subfigure}[t]{0.2\textwidth}
        \centering
        \begin{tikzpicture}
            \matrix[cmat] (mat) {
                2 &   &   \\
                1 & 1 &   \\
                2 & 1 & 3 \\
            };
        \end{tikzpicture}
        \caption{}
    \end{subfigure}\hfill
    \begin{subfigure}[t]{0.2\textwidth}
        \centering
        \begin{tikzpicture}
            \matrix[cmat] (mat) {
                1 &   & 1 \\
                  & 4 &   \\
                1 &   & 2 \\
            };
        \end{tikzpicture}
        \caption{}
    \end{subfigure}\hfill
    \begin{subfigure}[t]{0.2\textwidth}
        \centering
        \begin{tikzpicture}
            \matrix[cmat] (mat) {
                2 & 4 & 3 \\
                  &   &   \\
                  &   &   \\
            };
        \end{tikzpicture}
        \caption{}
    \end{subfigure}
    \caption{Examples of confusion matrices. (a) and (b) are both examples of clairvoyant classifiers. Although the classifier in (b) is wrong about the interpretation of the classes, the confusion matrix allows to infer the reality; for example, we immediately know that all the jobs labelled as class 2 are actually in class 1. (c) shows an example of a classifier that only over-estimates the jobs, and (d) is an example of a classifier that confuses class 1 and class 3 but is always correct for class 2. Finally, (e) shows a non-clairvoyant classifier.}
    \label{fig.cm-examples}
\end{figure}

In the following, we will sometimes restrict to the special case of \emph{binary classifiers}, that is to say, \(2\times 2\) confusion matrices. In this case, to simplify notations, we note \(s_k\) (resp.\ \(L_k\)) the number of small (resp.\ large) jobs in \(\mathcal{R}_k\), and we note \(n_s=s_1+s_2\) (resp.\ \(n_L=L_1+L_2\)) the total number of small (resp.\ large) jobs in the instance. Moreover, we note \(\lambda_s\) (resp.\ \(\lambda_L\)) the size of small (resp.\ large) jobs.

\medskip\noindent\textbf{Uncertainty model.} While the classifier allows to partition \(\mathcal{J}\) into \(K\) subsets \(\mathcal{R}_1,\mathcal{R}_2,\ldots,\mathcal{R}_K\), the permutation of the jobs in each \(\mathcal{R}_k\) remains completely unknown. In other words, the uncertainty of the problem is on the \emph{ordering} of each set \(\mathcal{R}_k\), and we do not make any assumption about it. We argue that, from a worst-case point of view, it seems completely unfair to compare to the optimal strategy of the clairvoyant case, which consists in scheduling jobs in non-decreasing order of processing times (the SPT policy): our decisions can always be subject to an adversary giving us large jobs first, whereas the optimal clairvoyant algorithm is never constrained by this malicious ordering. Hence, in the following of this paper, we compare to an optimal solution that is also subject to the fixed ordering of jobs in each set \(\mathcal{R}_k\). Our objective is to ensure robustness to the uncertain ordering of these sets.

In these settings, the sets \(\mathcal{R}_k\) may be seen as \emph{queues} from which an algorithm must pull jobs: at each decision step, the job that is in front of the chosen queue is immediately executed. We say that a \textbf{scenario} \(\sigma\) is a tuple of \(K\) permutations (also called \emph{chains}) \(\sigma_k\), \(1 \leq k \leq K\), where \(\sigma_k\) describes a fixed ordering of the queue \(\mathcal{R}_k\). The job that is in position \(j\) in the chain \(\sigma_k\) is denoted by \(\sigma_k[j]\), and for any \(1\le j_1\le j_2\le |\mathcal{R}_k|\), we note \(\sigma_k[j_1:j_2]\) the subchain of jobs \(\sigma_k[j_1],\sigma_k[j_1+1],\ldots,\sigma_k[j_2]\). Let \(\mathcal{S}\) denote the set of all possible scenarios.

We define a \textbf{sequence} \(r\) to be an ordered list \(r=(k_1,k_2,\ldots,k_n)\) of queues \(\mathcal{R}_k\) from which pulling and executing jobs. Note that a given queue may be chosen several times in a sequence, unless it contains a single job or is empty. For example, the sequence \((1,1,2,2,2,1,1,2,2)\) means that we execute the 2 first jobs of \(\mathcal{R}_1\), then the 3 first jobs of \(\mathcal{R}_2\), then the 2 next jobs of \(\mathcal{R}_1\), and finally the 2 next jobs of \(\mathcal{R}_2\). The queue that is chosen at step \(t\) in sequence \(r\) is noted \(r[t]\), and we note \(R\) the set of all possible sequences. An algorithm \(\mathcal{A}\) taking a scenario \(\sigma\) as input and returning a sequence \(r=\mathcal{A}(\sigma)\) is called a \textbf{sequencing algorithm}. Note that an optimal strategy is defined by a sequence optimizing some criterion, respecting however the actual scenario as for any sequencing algorithm.

The realization of a sequence \(r\) on a scenario \(\sigma\) generates the \textbf{schedule} \(\pi^r_{\sigma}\), which gives a permutation of the jobs, i.e., the \(j\)-th job in the schedule is denoted by \(\pi^r_{\sigma}[j]\). Note that this schedule defines a partition of each chain \(\sigma_k\) into subchains \(\sigma_k[1:j_1],\sigma_k[j_1+1:j_2],\sigma_k[j_2+1:j_3],\ldots\), such that the jobs of a given subchain are consecutive in \(\pi^r_\sigma\), and such that two subchains coming from the same chain \(\sigma_k\) are never consecutive in \(\pi^r_\sigma\). We note \(\mathcal{P}^r_{\sigma_k}\) the ordered set of subchains of \(\sigma_k\) induced by the realization of the sequence \(r\) on scenario \(\sigma\).

\medskip\noindent\textbf{Robustness criteria.} For any given sequence \(r\) and scenario \(\sigma\), the sum of completion times of the schedule \(\pi^r_{\sigma}\) is defined as \(\tilde{C}^r_{\sigma}=\sum_{j=1}^n (n-j+1)p_{\pi^r_{\sigma}[j]}\). By slightly abusing the notation, \(\pi^{\mathcal{A}}_\sigma\) (resp.\ \(\tilde{C}^{\mathcal{A}}_\sigma\)) denotes the schedule (resp.\ objective value) obtained by executing algorithm \(\mathcal{A}\) on scenario \(\sigma\). The sequence that minimizes the objective for a given scenario \(\sigma\) is called the \textbf{\(\sigma\)-optimal sequence} and is noted \(r^*_\sigma\). The corresponding optimal objective value is simply noted \(\tilde{C}^*_{\sigma}\). We can immediately observe that when the scenario \(\sigma\) is known, it is straightforward to compute \(r^*_\sigma\), as this minimization problem is a special case of \(1|\text{chains}|\sum C_j\). Indeed, one can view the fixed ordering of jobs of a given queue \(\mathcal{R}_k\) as a chain of precedence relationships. This problem is polynomially solvable, as described in Theorem 4-2 of~\cite{conway1967theory}: compute the average processing time of each prefix of each chain \(\sigma_k\) and schedule completely the prefix of lowest cost; repeat (ignoring already scheduled jobs) until there is no remaining job in the chains. We call this algorithm \textsc{Chain-Opt} (see Appendix~\ref{sec.chain-opt} for details), and we assume that the given solution always maximizes the number of jobs scheduled at each step.

We propose to study algorithmic robustness for the uncertain scheduling problem (where the permutations of the queues \(\mathcal{R}_k\) are unknown) under different settings. In robust optimization with multiple scenarios, it is common to study the so-called Min-Max and Min-Max Regret problems, where the former seeks to find the solution whose worst-case scenario gives the lowest objective value, while the latter seeks to find the solution whose worst-case scenario gives the lowest absolute deviation from the optimal. Note that our problem differs from the classical framework in the nature of the scenarios: they are usually defined as sets of numeric values, representing the uncertain parameters, while our scenarios consist in the set of all possible permutations of the queues \(\mathcal{R}_k\). More formally, we respectively define the \textsc{Min-Max} and \textsc{Min-Max Regret} problems as following:
\begin{equation}
    \label{pb.min-max}
    \tag{\textsc{Min-Max}}
    \min_{\mathcal{A}}\max_{\sigma\in\mathcal{S}}\left\{\tilde{C}^{\mathcal{A}}_{\sigma}\right\},
\end{equation}
\begin{equation}
    \label{pb.min-max-regret}
    \tag{\textsc{Min-Max Regret}}
    \min_{\mathcal{A}}\max_{\sigma\in\mathcal{S}}\left\{\tilde{C}^{\mathcal{A}}_{\sigma}-\tilde{C}^*_\sigma\right\},
\end{equation}
where the minimum is taken over a different set of sequencing algorithms \(\mathcal{A}\) according to the situation. Indeed, we study two kinds of algorithms in the following: \textbf{non-adaptive} algorithms \(\mathsf{NA}\), which must return the full sequence before any execution of the jobs, and \textbf{adaptive} algorithms \(\mathsf{A}\supset\mathsf{NA}\), which may adapt their choices during the execution of the schedule. While the former are easier to analyze, the latter make a malicious adversary less powerful. We also consider \textbf{randomized} algorithms, which are allowed to randomly sample a sequence, and we adapt the metric accordingly, i.e., the \textsc{Min-Max Regret} problem becomes
\begin{equation}
    \label{pb.min-max-regret-rand}
    \tag{\textsc{Randomized Min-Max Regret}}
    \min_{\mathcal{A}}\max_{\sigma\in\mathcal{S}}\left\{\E\left[\tilde{C}^{\mathcal{A}}_{\sigma}\right]-\tilde{C}^*_\sigma\right\},
\end{equation}
where the expectation is taken over the choices of algorithm \(\mathcal{A}\). Finally, we also consider in similar settings the \textsc{Min-Expectation} problem, which seeks to minimize the expected sum of completion times when each scenario \(\sigma\) has probability \(\Prob(\sigma)\) to appear:
\begin{equation}
    \label{pb.min-expectation}
    \tag{\textsc{Min-Expectation}}
    \min_{\mathcal{A}}\E\left[\tilde{C}^{\mathcal{A}}\right]
    =\min_{\mathcal{A}}\sum_{\sigma\in\mathcal{S}}\Prob(\sigma)\tilde{C}^{\mathcal{A}}_{\sigma}.
\end{equation}
Unless specified otherwise, we assume that all scenarios are equiprobable, i.e., \(\Prob(\sigma)=1/|\mathcal{S}|\).

\medskip\noindent\textbf{Contributions and organization.} The paper is organized as follows. Section~\ref{sec.related-work} is devoted to related work. In Section~\ref{sec.non-adaptive}, we give a very simple algorithm, called \textsc{Min-Cost}, and we prove that it is the best possible \emph{non-adaptive} algorithm for the three considered robustness criteria (Theorems~\ref{th.min-max}-\ref{th.min-max-regret-lb}). Interestingly, we demonstrate that adaptive or randomized algorithms cannot do better for the \textsc{Min-Max} problem (Lemma~\ref{lem.min-max-dec-scenario}). We also show that, for binary classifiers, no deterministic adaptive algorithm can perform better than \textsc{Min-Cost} for the \textsc{Min-Max Regret} problem (Lemma~\ref{lem.min-max-regret-lb}). However, we prove in Section~\ref{sec.beyond-non-adaptive} that for \textsc{Min-Max Regret} and binary classifiers, a simple randomized algorithm outperforms \textsc{Min-Cost} for specific structures in the confusion matrix (Theorem~\ref{th.min-max-regret-rand}). Moreover, for \textsc{Min-Expectation}, we give an optimal adaptive dynamic program, which is polynomial when the number of classes \(K\) and the entries of the confusion matrix are bounded above by a constant (Theorem~\ref{thm:dp-optimal}). Finally, by giving a specific family of classifiers, we show that the relative gap between \textsc{Min-Cost} (i.e., optimal non-adaptive) and this optimal adaptive algorithm is asymptotically bigger than $4/3$ (Theorem~\ref{th:nonadapt-vs-adapt}).

\section{Related Work}
\label{sec.related-work}

In the offline case where all jobs are fully known in advance, $1||\sum C_j$ can be solved in polynomial time by applying the Shortest Processing Time (SPT) policy~\cite{pinedo2012scheduling}.
In the online case where the jobs are revealed according to an adversarial order (online over list), the LPT order corresponds to a worst-case scenario for any no-idling algorithm. Moreover, in the case where the instance is composed by only two classes of jobs with processing times $\lambda_s$, $\lambda_L$, with $\lambda_s<\lambda_L$, considering a random ordering leads to an expected competitive ratio of $O\left(\sqrt{\lambda_L/\lambda_s}\right)$~\cite{mitzenmacher2020scheduling}.
In the online case where the jobs arrive over time, Hoogeveen and Vestjens~\cite{HoogeveenV96} showed that no deterministic algorithm can have a competitive ratio better than 2, and they proposed the delayed-SPT algorithm which attains this ratio.

In the non-clairvoyant model where the processing time of each job becomes known only at its completion, no (randomized) algorithm can achieve a constant competitive ratio~\cite{JagerSWW25}.
In order to deal with non-clairvoyant data, several works consider the learning-augmented setting, where non-necessarily accurate predictions are provided.
Kumar et al.~\cite{purohit2018improving} were the first to study the sum of completion times objective when preemption of jobs is permitted and predictions on the processing times are available, and they proposed an algorithm being robust (bounded performance for any quality of predictions) and consistent (good performance on accurate predictions).
Since then, a series of papers has appeared, mainly handling the preemptive case~\cite{BampisDKLP22,benomar2026nonclairvoyant,im2023non,WeiZ20/nips}.
Closer to our setting but still in the preemptive case, Lindermayr and Megow~\cite{LindermayrM25} consider predictions on jobs permutations implying a priority order among the jobs. Our framework differs in that the uncertainty is captured by a full confusion matrix rather than a single priority order.

Another way to circumvent the non-clairvoyant model is to treat processing times as random variables whose distributions (or at least their first moments) are known in advance. Then, the objective is to minimize the expected sum of completion times \(\E\left[\sum C_j\right]\). On a single machine, the Shortest Expected Processing Time (SEPT) policy, which executes jobs in non-decreasing order of \(\E[p_j]\), is optimal, and this is also the case on parallel machines if processing times follow an exponential distribution~\cite{pinedo2012scheduling}. For general distributions, the approximation ratio of SEPT is \(\Omega(n^{1/4})\)~\cite{im2015stochastic}, but is almost linear in the squared coefficient of variation of processing times, which indicates that SEPT performs well if there is not much variability among jobs~\cite{mohring1999approximation}. The dependence on this squared coefficient of variation is also characteristic of several proposed LP-based strategies~\cite{mohring1999approximation,skutella2014stochastic}. Other works consider the online stochastic model, in which the processing times are random variables and jobs stay unknown before their arrival in the system~\cite{megow2006models,schulz2008stochastic}.

The problem $1||\sum C_j$ has also been studied in the classical robust optimization setting~\cite{KouvelisY97}.
In order to represent the uncertain information, in the \textit{discrete uncertainty} model, a finite set of different realization scenarios for the processing times of the jobs is available, while in the \textit{interval uncertainty} model, only a lower and an upper bound of them are provided.
In both cases, commonly used metrics include Min-Max, Min-Max Regret and Min-Max Relative Regret, similarly to the metrics used in our work. All of them seek to guarantee a robust solution, that is a solution with good performance even in the worst-case realization of the input parameters.
Yang and Yu~\cite{YangY02} have proved that $1||\sum C_j$ under discrete uncertainty is NP-hard for all the three metrics.
For the Min-Max criterion, Mastrolilli et al.~\cite{mastrolilli2013single} have shown that it is NP-hard to approximate it within a factor less than $6/5$, and they give a $2$-approximation algorithm for the weighted case with precedence constraints.
Moreover, they propose a dynamic programming algorithm running in polynomial time if the number of scenarios and the values of processing times are bounded by a constant.
When considering interval uncertainty, both the Min-Max Regret~\cite{LebedevA06} and the Min-Max Relative Regret~\cite{Averbakh05} are known to be NP-hard.
Note that our work introduces a new model to represent uncertain data, which is of a completely different nature with respect to the known ones.

Kahan~\cite{Kahan91} proposed a model in which an oracle is queried to reveal the real values of some uncertain data at a given cost.
Adapting this idea in scheduling problems, D\"urr et al.~\cite{DurrEMM20} consider the query to be an additional job to be executed on the same machines as the other jobs.
Querying a job is decided by the scheduler and may lead to a reduction of its processing time.
Without a query, the processing time of a job is equal to a given upper bound.
In this context, they propose deterministic and randomized results in different settings of $1||\sum C_j$.
Moreover, adaptive and non-adaptive algorithms in this model are proposed in \cite{DufosseDNTV20} for instances consisting of two types of jobs.
Finally, Bampis et al.~\cite{BampisE0PX25} proposed a probabilistic prediction model for permutation problems, which tries to limit the number of predictions used to optimize an objective. The authors showed that their model also applies to $1||\sum C_j$ for which any no-idling solution corresponds to a permutation of jobs.

\section{An Oblivious Non-Adaptive Algorithm}
\label{sec.non-adaptive}

We begin by studying a non-adaptive sequencing algorithm called \textsc{Min-Cost} (\MC), which schedules each \(\mathcal{R}_k\) completely in non-decreasing order of average processing time
\[
    \bar{p}_k=\frac{1}{|\mathcal{R}_k|}\sum_{j\in\mathcal{R}_k}p_j=\frac{1}{|\mathcal{R}_k|}\sum_{\ell=1}^K e_{k\ell}\lambda_\ell.
\]
To make the analysis easier, we assume in the following that the rows of any confusion matrix are arranged in non-decreasing order of \(\bar{p}_k\): if this is not the case, it suffices to rename the rows accordingly. This way, \textsc{Min-Cost} simply consists in executing \(\mathcal{R}_1\) completely, then \(\mathcal{R}_2\), and so on.

We will show that \textsc{Min-Cost} is the best possible \emph{deterministic} algorithm for the \textsc{Min-Max} and \textsc{Min-Max Regret} problems (Theorems~\ref{th.min-max}, \ref{th.min-max-regret-ub} and \ref{th.min-max-regret-lb} in Sections~\ref{sec.min-max} and \ref{sec.min-max-regret}), and that it is the best possible \emph{non-adaptive} algorithm for the \textsc{Min-Expectation} problem (Theorem~\ref{th.min-expectation} in Section~\ref{sec.min-expectation}). In this sense, we call \textsc{Min-Cost} an \emph{oblivious} non-adaptive algorithm, as it solves three robust problems at once for the class of algorithms \(\mathsf{NA}\).

\subsection{\textsc{Min-Max} problem}
\label{sec.min-max}

\begin{theorem}
    \label{th.min-max}
    \textsc{Min-Cost} is an optimal algorithm for the \textsc{Min-Max} problem.
\end{theorem}

To prove Theorem~\ref{th.min-max}, we state a simple lemma showing that the worst scenario for any possible solution consists in arranging all queues \(\mathcal{R}_k\) in non-increasing order of processing times.

\begin{lemma}
    \label{lem.min-max-dec-scenario}
    Let \(\texttt{DEC}\) be the scenario where each queue \(\mathcal{R}_k\) is arranged in non-increasing order of processing times. Then, the scenario \(\texttt{DEC}\) maximizes the objective for any sequence \(r\).
\end{lemma}
\begin{proof}
    Let \(r\) be an arbitrary sequence and \(\hat{\sigma}\) a scenario that maximizes the objective of \(r\). Suppose there exists a queue \(\mathcal{R}_k\) that is not arranged in non-increasing order of processing times, i.e., there exists a position \(1\le j<|\mathcal{R}_k|\) such that \(p_{\hat{\sigma}_k[j]}<p_{\hat{\sigma}_k[j+1]}\). These two jobs \(\hat{\sigma}_k[j]\) and \(\hat{\sigma}_k[j+1]\) are respectively scheduled during steps \(t_1\) and \(t_2\), with \(t_1<t_2\). As they are consecutive in \(\mathcal{R}_k\), we have \(r[t_1]=r[t_2]=k\) and \(r[\tau]\ne k\) for all \(t_1<\tau<t_2\). Thus, swapping jobs \(\hat{\sigma}_k[j]\) and \(\hat{\sigma}_k[j+1]\) in \(\hat{\sigma}_k\) will also swap them in the schedule \(\pi^r_{\hat{\sigma}}\) without altering the position of the other jobs, and this operation will necessarily increase the sum of completion times. As \(\hat{\sigma}\) was chosen to maximize this objective, we get a distinct scenario that also maximizes it. By repeatedly swapping consecutive unordered jobs in each set \(\mathcal{R}_k\), we obtain the scenario \(\texttt{DEC}\).
\end{proof}

According to this lemma, it suffices to build the \(\texttt{DEC}\)-optimal sequence to solve the \textsc{Min-Max} problem. The solution provided by \textsc{Chain-Opt} for \(\texttt{DEC}\) is in fact equivalent to scheduling the queues \(\mathcal{R}_k\) in non-decreasing order of average processing time, which is exactly the \textsc{Min-Cost} procedure.

\begin{proof}[Theorem~\ref{th.min-max}]
    Apply \textsc{Chain-Opt} on \(\texttt{DEC}\). As the jobs are ordered in non-increasing order of processing times in each \(\mathcal{R}_k\), this implies that the average processing time of the full queue is smaller than the average processing time of any shorter prefix of this queue. Thus, \textsc{Chain-Opt} schedules queues in non-decreasing order of average processing time \(\bar{p}_k\), exactly like \textsc{Min-Cost}. Note that Lemma~\ref{lem.min-max-dec-scenario} applies to any solution, no matter how the sequence is built, which proves that no adaptive or randomized algorithm can do better than \textsc{Min-Cost}.
\end{proof}

\subsection{\textsc{Min-Expectation} problem}
\label{sec.min-expectation}

For non-adaptive algorithms and when all scenarios are equiprobable, the \textsc{Min-Expectation} problem is straightforward to solve by computing the expected processing time of each job.

\begin{theorem}
    \label{th.min-expectation}
    If all scenarios are equiprobable, \textsc{Min-Cost} is an optimal non-adaptive algorithm for the \textsc{Min-Expectation} problem.
\end{theorem}
\begin{proof}
    The \textsc{Min-Expectation} problem can be rewritten as
    \begin{align*}
        \min_{\mathcal{A}}\sum_{\sigma\in\mathcal{S}}\Prob(\sigma)\tilde{C}^{\mathcal{A}}_{\sigma}
            &=\min_{\mathcal{A}}\sum_{\sigma\in\mathcal{S}}\Prob(\sigma)\left(\sum_{j=1}^{n}(n-j+1)p_{\pi^{\mathcal{A}}_\sigma[j]}\right) \\
            &=\min_{\mathcal{A}}\sum_{j=1}^{n}(n-j+1)\left(\sum_{\sigma\in\mathcal{S}}\Prob(\sigma)p_{\pi^{\mathcal{A}}_\sigma[j]}\right),
    \end{align*}
    where the term \(\sum_{\sigma\in\mathcal{S}}\Prob(\sigma)p_{\pi^{\mathcal{A}}_\sigma[j]}\) is the expected processing time of the \(j\)-th job in the schedule built by algorithm \(\mathcal{A}\) over all scenarios. Clearly, if \(\mathcal{A}\) is a non-adaptive algorithm, the job \(\pi^{\mathcal{A}}_\sigma[j]\) always comes from the same queue \(\mathcal{R}_k\), no matter the scenario. Thus, if all scenarios are equiprobable, we have
    \[
        \sum_{\sigma\in\mathcal{S}}\Prob(\sigma)p_{\pi^{\mathcal{A}}_\sigma[j]}=\frac{1}{|\mathcal{R}_k|}\sum_{j'\in\mathcal{R}_k}p_{j'}=\bar{p}_k.
    \]
    It is well-known that the optimal solution consists in scheduling jobs in non-decreasing order of expected processing time (see for example Theorem~10.1.1 in~\cite{pinedo2012scheduling}), which is exactly \textsc{Min-Cost} in this case.
\end{proof}

\subsection{\textsc{Min-Max Regret} problem}
\label{sec.min-max-regret}

The \textsc{Min-Max Regret} problem is more challenging to analyze. First, we derive an upper bound for the maximum regret of \textsc{Min-Cost}. Then, we derive a matching lower bound for the maximum regret of any deterministic algorithm in the special case of 2 classes of jobs (binary classifiers). This shows that \textsc{Min-Cost} is the best possible deterministic algorithm for the \textsc{Min-Max Regret} problem on 2 classes, and we conjecture that this result holds for an arbitrary number of classes.

\begin{theorem}
    \label{th.min-max-regret-ub}
    For any instance, we have
    \[
        \max_{\sigma\in\mathcal{S}}\left\{\tilde{C}^{\MC}_{\sigma}-\tilde{C}^*_\sigma\right\}\le\sum_{1\le k<\ell\le K}(\lambda_\ell-\lambda_k)\sum_{1\le i<j\le K}e_{jk}e_{i\ell}.
    \]

    \noindent In particular, \(\max_{\sigma\in\mathcal{S}}\left\{\tilde{C}^{\MC}_{\sigma}-\tilde{C}^*_\sigma\right\}\le s_2 L_1(\lambda_L-\lambda_s)\) for any binary classifier.
\end{theorem}

\begin{theorem}
    \label{th.min-max-regret-lb}
    For binary classifiers, we have \(\min_{\mathcal{A}\in\mathsf{A}}\max_{\sigma\in\mathcal{S}}\left\{\tilde{C}^{\mathcal{A}}_{\sigma}-\tilde{C}^*_\sigma\right\}\ge s_2 L_1(\lambda_L-\lambda_s)\).
\end{theorem}

We prove the upper bound in Theorem~\ref{th.min-max-regret-ub} by characterizing the worst-case scenario for \textsc{Min-Cost} (i.e., the scenario that maximizes the regret of \textsc{Min-Cost}). Then, we compute the optimal solution for this worst-case scenario by applying \textsc{Chain-Opt}, which then allows us to derive the expression of the regret.

The following lemma shows a useful general property on the worst-case scenario of any sequence in the \textsc{Min-Max Regret} problem. A direct implication of this lemma is that the worst-case scenario for \textsc{Min-Cost} is when all queues are arranged in non-decreasing order of processing times (in contrast with the worst-case scenario in the \textsc{Min-Max} problem, which was exactly the inverse ordering).

\begin{lemma}
    \label{lem.worst-case-min-max-regret}
    For any sequence \(r\), the set of worst-case scenarios necessarily includes the scenario \(\hat{\sigma}\) where, for all chains \(\hat{\sigma}_k\), each sub-chain \(\hat{\sigma}_k[j_1:j_2]\in\mathcal{P}^r_{\hat{\sigma}_k}\) is non-decreasing.
\end{lemma}
\begin{proof}
    Let \(r\) be an arbitrary sequence, and let \(\hat{\sigma}\) be a scenario that maximizes the regret of \(r\). Suppose there is a chain \(\hat{\sigma}_k\) such that there exists a sub-chain \(\hat{\sigma}_k[j_1:j_2]\in\mathcal{P}^r_{\hat{\sigma}_k}\) that is not arranged in non-decreasing order of processing times, i.e., there exists a position \(j_1\le j<j_2\) such that \(p_{\hat{\sigma}_k[j]}>p_{\hat{\sigma}_k[j+1]}\). Consider the alternative scenario \(\hat{\sigma}'\) where jobs \(\hat{\sigma}_k[j]\) and \(\hat{\sigma}_k[j+1]\) are swapped. As these jobs are consecutive in the schedule \(\pi^r_{\hat{\sigma}}\) (by definition), we get
    \[
        \tilde{C}^r_{\hat{\sigma}'}
            =\tilde{C}^r_{\hat{\sigma}}-\left(p_{\hat{\sigma}_k[j]}-p_{\hat{\sigma}_k[j+1]}\right).
    \]
    Now let us consider the \(\hat{\sigma}\)-optimal sequence \(r^*_{\hat{\sigma}}\). If we only swap jobs \(\hat{\sigma}_k[j]\) and \(\hat{\sigma}_k[j+1]\), leaving the other jobs untouched in the schedule \(\pi^*_{\hat{\sigma}}\), we necessarily get
    \[
        \tilde{C}^{r^*_{\hat{\sigma}}}_{\hat{\sigma}'}\le \tilde{C}^*_{\hat{\sigma}}-\left(p_{\hat{\sigma}_k[j]}-p_{\hat{\sigma}_k[j+1]}\right).
    \]
    Then, as \(\tilde{C}^*_{\hat{\sigma}'}\le\tilde{C}^{r^*_{\hat{\sigma}}}_{\hat{\sigma}'}\), we conclude \(\tilde{C}^*_{\hat{\sigma}}-\tilde{C}^*_{\hat{\sigma}'}\ge p_{\hat{\sigma}_k[j]}-p_{\hat{\sigma}_k[j+1]}=\tilde{C}^r_{\hat{\sigma}}-\tilde{C}^r_{\hat{\sigma}'}\), that is, \(\tilde{C}^r_{\hat{\sigma}'}-\tilde{C}^*_{\hat{\sigma}'}\ge\tilde{C}^r_{\hat{\sigma}}-\tilde{C}^*_{\hat{\sigma}}\). In other words, swapping two consecutive jobs arranged in decreasing order of processing times in a given sub-chain can only increase the regret of \(r\). By successive swaps, the wanted result follows.
\end{proof}

As \textsc{Min-Cost} executes each queue \(\mathcal{R}_k\) completely one after the other, we necessarily have \(\mathcal{P}^{\MC}_{\hat{\sigma}_k}=\left\{\hat{\sigma}_k[1:|\mathcal{R}_k|]\right\}\) for all \(k\), which by Lemma~\ref{lem.worst-case-min-max-regret} implies that the worst-case scenario for \textsc{Min-Cost} is when each queue is non-decreasing. Let \texttt{INC} denote this scenario. It is easy to see that \textsc{Chain-Opt} will actually provide an SPT schedule for \texttt{INC}. Hence, we can directly compute the maximum regret of \textsc{Min-Cost}, which necessarily constitute an upper bound for the optimal solution of the \textsc{Min-Max Regret} problem.

\begin{proof}[Theorem~\ref{th.min-max-regret-ub}]
    Recall that we assumed \(\bar{p}_1\le\cdots\le\bar{p}_K\). Hence, for the scenario \texttt{INC}, \textsc{Min-Cost} and \textsc{Chain-Opt} will produce the following schedules, where each \(e_{k\ell}\) represents the contiguous segment of jobs \(\mathcal{R}_k\cap\mathcal{J}_\ell\):
    
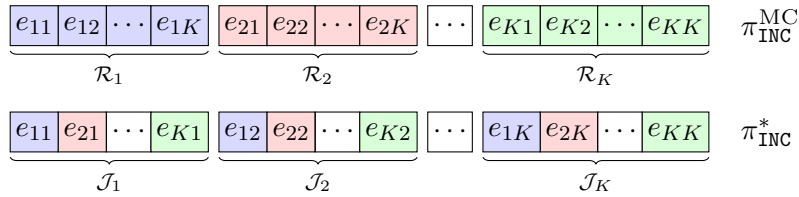
    \begin{center}
        \begin{tikzpicture}[
            seg/.style={
                inner xsep=2pt,
                minimum height=15pt,
                draw
            },
            lab/.style={
                midway,
                below,
                yshift=-6pt,
                font=\scriptsize
            },
        ]
            \begin{scope}
                \node[seg,fill=blue!15!white] (e11) at (0,0) {\(e_{11}\)};
                \node[seg,right=0 of e11,fill=blue!15!white] (e12) {\(e_{12}\)};
                \node[seg,right=0 of e12,fill=blue!15!white] (e1dots) {\(\cdots\)};
                \node[seg,right=0 of e1dots,fill=blue!15!white] (e1K) {\(e_{1K}\)};
                
                \node[seg,right=4pt of e1K,fill=red!15!white] (e21) {\(e_{21}\)};
                \node[seg,right=0 of e21,fill=red!15!white] (e22) {\(e_{22}\)};
                \node[seg,right=0 of e22,fill=red!15!white] (e2dots) {\(\cdots\)};
                \node[seg,right=0 of e2dots,fill=red!15!white] (e2K) {\(e_{2K}\)};
                
                \node[seg,right=4pt of e2K] (dots) {\(\cdots\)};
                
                \node[seg,right=4pt of dots,fill=green!15!white] (eK1) {\(e_{K1}\)};
                \node[seg,right=0 of eK1,fill=green!15!white] (eK2) {\(e_{K2}\)};
                \node[seg,right=0 of eK2,fill=green!15!white] (eKdots) {\(\cdots\)};
                \node[seg,right=0 of eKdots,fill=green!15!white] (eKK) {\(e_{KK}\)};
                
                \draw[decoration={brace,mirror},decorate] ([yshift=-2pt]e11.south west) --
                    node[lab] {\(\mathcal{R}_1\)} ([yshift=-2pt]e1K.south east);
                \draw[decoration={brace,mirror},decorate] ([yshift=-2pt]e21.south west) --
                    node[lab] {\(\mathcal{R}_2\)} ([yshift=-2pt]e2K.south east);
                \draw[decoration={brace,mirror},decorate] ([yshift=-2pt]eK1.south west) --
                    node[lab] {\(\mathcal{R}_K\)} ([yshift=-2pt]eKK.south east);
                
                \node[right=12pt of eKK] {\(\pi^{\MC}_{\texttt{INC}}\)};
            \end{scope}
            \begin{scope}[yshift=-40pt]
                \node[seg,fill=blue!15!white] (e11) at (0,0) {\(e_{11}\)};
                \node[seg,right=0 of e11,fill=red!15!white] (e12) {\(e_{21}\)};
                \node[seg,right=0 of e12] (e1dots) {\(\cdots\)};
                \node[seg,right=0 of e1dots,fill=green!15!white] (e1K) {\(e_{K1}\)};
                
                \node[seg,right=4pt of e1K,fill=blue!15!white] (e21) {\(e_{12}\)};
                \node[seg,right=0 of e21,fill=red!15!white] (e22) {\(e_{22}\)};
                \node[seg,right=0 of e22] (e2dots) {\(\cdots\)};
                \node[seg,right=0 of e2dots,fill=green!15!white] (e2K) {\(e_{K2}\)};
                
                \node[seg,right=4pt of e2K] (dots) {\(\cdots\)};
                
                \node[seg,right=4pt of dots,fill=blue!15!white] (eK1) {\(e_{1K}\)};
                \node[seg,right=0 of eK1,fill=red!15!white] (eK2) {\(e_{2K}\)};
                \node[seg,right=0 of eK2] (eKdots) {\(\cdots\)};
                \node[seg,right=0 of eKdots,fill=green!15!white] (eKK) {\(e_{KK}\)};
                
                \draw[decoration={brace,mirror},decorate] ([yshift=-2pt]e11.south west) --
                    node[lab] {\(\mathcal{J}_1\)} ([yshift=-2pt]e1K.south east);
                \draw[decoration={brace,mirror},decorate] ([yshift=-2pt]e21.south west) --
                    node[lab] {\(\mathcal{J}_2\)} ([yshift=-2pt]e2K.south east);
                \draw[decoration={brace,mirror},decorate] ([yshift=-2pt]eK1.south west) --
                    node[lab] {\(\mathcal{J}_K\)} ([yshift=-2pt]eKK.south east);
                
                \node[right=12pt of eKK] {\(\pi^*_{\texttt{INC}}\)};
            \end{scope}
        \end{tikzpicture}
    \end{center}
    
    When transitioning from \(\pi^{\MC}_{\texttt{INC}}\) to \(\pi^*_{\texttt{INC}}\), we see that, for each \(k,\ell\) such that \(1\le k<\ell\le K\), the \(e_{k\ell}\) jobs \(\mathcal{R}_k\cap\mathcal{J}_\ell\) are moved after the \(e_{ij}\) jobs \(\mathcal{R}_i\cap\mathcal{J}_j\) for each \(k<i\le K\) and \(1\le j<\ell\). Hence, the contribution of each segment \(e_{k\ell}\) to \(\tilde{C}^{\MC}_{\texttt{INC}}-\tilde{C}^*_{\texttt{INC}}\) is \(e_{k\ell}\sum_{j=1}^{\ell-1}(\lambda_\ell-\lambda_j)\sum_{i=k+1}^K e_{ij}\). Summing over \(1\le k<\ell\le K\) and rearranging the terms yields
    \[
        \tilde{C}^{\MC}_{\texttt{INC}}-\tilde{C}^*_{\texttt{INC}}
            =\sum_{\mathclap{1\le k<\ell\le K}}\; (\lambda_\ell-\lambda_k)
                \sum_{\mathclap{1\le i<j\le K}}\; e_{jk}e_{i\ell}.
    \]
    By Lemma~\ref{lem.worst-case-min-max-regret}, \texttt{INC} is the worst-case scenario for \textsc{Min-Cost}, hence the theorem follows.
\end{proof}

For the lower bound, Theorem~\ref{th.min-max-regret-lb} is a direct consequence of the following lemma.

\begin{lemma}
    \label{lem.min-max-regret-lb}
    For any binary classifier and any deterministic algorithm \(\mathcal{A}\), there exists a scenario \(\hat{\sigma}\) such that \mbox{\(\tilde{C}^{\mathcal{A}}_{\hat{\sigma}}-\tilde{C}^*_{\hat{\sigma}}\ge s_2L_1(\lambda_L-\lambda_s)\)}.
\end{lemma}
\begin{proof}
Any deterministic algorithm \(\mathcal{A}\) must schedule the \(s_1+L_1\) jobs of \(\mathcal{R}_1\) and the \(s_2+L_2\) jobs of \(\mathcal{R}_2\), and its only degree of freedom is to decide at each step which queue to draw from. We construct an adversarial scenario as follows: whenever \(\mathcal{A}\) draws a job from any queue \(\mathcal{R}_k\), the scenario reveals it to be large, until one of the two queues has exhausted all its large jobs.

First, assume the queue \(\mathcal{R}_1\) exhausts its large jobs before \(\mathcal{R}_2\), i.e., \(\mathcal{A}\) has drawn \(x<L_2\) large jobs from \(\mathcal{R}_2\) when all \(L_1\) large jobs of \(\mathcal{R}_1\) have been scheduled. At this point, the remaining jobs are \(s_1\) small jobs in \(\mathcal{R}_1\), and \(s_2\) small and \(L_2-x\) large jobs in \(\mathcal{R}_2\). Furthermore, the adversarial scenario fixes the rest of internal order of \(\mathcal{R}_2\) (the first \(x\) jobs have already been decided large): the \(s_2\) small jobs, followed by the remaining \(L_2-x\) large jobs. 

After the reveal of the \(L_1\) large jobs of \(\mathcal{R}_1\), the best thing that \(\mathcal{A}\) could do is to schedule all remaining jobs of \(\mathcal{R}_1\) (known as small, as the first \(L_1\) jobs are large), then finish \(\mathcal{R}_2\). This yields the schedule \(\underbrace{\mathsf{L}\cdots\mathsf{L}}_{x+L_1}\ \underbrace{\mathsf{s}\cdots\mathsf{s}}_{s_1} \ \underbrace{\mathsf{s}\cdots\mathsf{s}}_{s_2}\ \underbrace{\mathsf{L}\cdots\mathsf{L}}_{L_2-x}\), where \(\mathsf{L}\) (resp.\ \(\mathsf{s}\)) denotes a large (resp.\ small) job. However, consider the alternative schedule that first processes the \(x+s_2\) jobs of \(\mathcal{R}_2\), then schedules all jobs from \(\mathcal{R}_1\), and finally the remaining \(L_2-x\) jobs from \(\mathcal{R}_2\). This yields \(\underbrace{\mathsf{L}\cdots\mathsf{L}}_{x}\ \underbrace{\mathsf{s}\cdots\mathsf{s}}_{s_2} \ \underbrace{\mathsf{L}\cdots\mathsf{L}}_{L_1}\ \underbrace{\mathsf{s}\cdots\mathsf{s}}_{s_1}\ \underbrace{\mathsf{L}\cdots\mathsf{L}}_{L_2-x}\).
A direct computation shows that the cost difference between these two schedules is \(s_2L_1\left(\lambda_L-\lambda_S\right)\), and the difference with the optimal schedule can only be greater.

In the symmetric case, where queue \(\mathcal{R}_2\) exhausts its large jobs first, the same argument applies with queues \(\mathcal{R}_1\) and \(\mathcal{R}_2\) exchanged, yielding a cost difference of \(s_1L_2\left(\lambda_L-\lambda_S\right)\). Recall that, by assumption, we have \(\bar{p}_1\le\bar{p}_2\), that is, \(\frac{s_1\lambda_s+L_1\lambda_L}{s_1+L_1}\le\frac{s_2\lambda_s+L_2\lambda_L}{s_2+L_2}\), i.e., \(s_2L_1\le s_1L_2\). Thus, the bound also holds in this case.
\end{proof}

\section{Beyond Non-Adaptivity}
\label{sec.beyond-non-adaptive}

Although we have shown there is no hope to obtain a better solution than \textsc{Min-Cost} for the \textsc{Min-Max} problem, we demonstrate in this section that randomization helps for the \textsc{Min-Max Regret} problem. In particular, we give a randomized algorithm, \textsc{Rand-Half} (\RH), which breaks the lower bound given in Lemma~\ref{lem.min-max-regret-lb} if the confusion matrix exhibits a specific structure. Moreover, we will show in the case of the \textsc{Min-Expectation} problem that adaptive algorithms generally do better than \textsc{Min-Cost}, and we prove that the relative gap between an optimal non-adaptive algorithm and an optimal adaptive algorithm is asymptotically greater than 4/3.

\subsection{Randomization for \textsc{Min-Max Regret}}

We show that, in the case of 2 classes (binary classifiers), it is possible to break the lower bound for deterministic algorithms by allowing to sample a solution randomly before executing any job. We consider the \textsc{Randomized Min-Max Regret} criterion, where the adversarial scenario must be chosen without knowing anything about the random number generator of the algorithm. We analyze a randomized algorithm called \textsc{Rand-Half} (\RH), which chooses a solution randomly among 3 different sequences. We show how to set the probability of each sequence in order to minimize the value of the maximum regret of \textsc{Rand-Half}, yielding the following theorem.

\begin{theorem}
    \label{th.min-max-regret-rand}
    For any binary classifier, we have \(\displaystyle\max_{\sigma\in\mathcal{S}}\left\{\E\left[\tilde{C}^{\RH}_{\sigma}\right]-\tilde{C}^*_\sigma\right\}\le (s_2L_1-\delta)(\lambda_2-\lambda_1)\), where
    \[
        \delta=\begin{cases}
            \frac{s_2^2(L_1-s_1)^2}{L_1L_2-s_1s_2} & \text{if }s_1<L_1\text{ and }s_2\le L_2,\text{or }s_1\le L_1\text{ and }s_2<L_2, \\[4pt]
            \frac{L_1^2(s_2-L_2)^2}{s_1s_2-L_1L_2} & \text{if }L_1<s_1\text{ and }L_2\le s_2,\text{or }L_1\le s_1\text{ and }L_2<s_2, \\[4pt]
            0 & \text{otherwise.} \\
        \end{cases}
    \]
\end{theorem}
\begin{remark}
    Note that, for binary classifiers, the problem is trivial if there is at least one zero entry in the confusion matrix, as it allows to know whether a given queue contains only small or only large jobs. Thus, we consider in the following that all entries are nonzero.
\end{remark}

The \textsc{Rand-Half} algorithm randomly chooses a solution among the 3 following sequences:
\begin{description}
    \item[Sequence \(r_1\)] The jobs \(\sigma_1[1:|\mathcal{R}_1|]\), followed by \(\sigma_2[1:|\mathcal{R}_2|]\), which is equivalent to \textsc{Min-Cost}.
    \item[Sequence \(r_2\)] The jobs \(\sigma_1[1:s_1]\), followed by \(\sigma_2[1:s_2]\), then \(\sigma_1[s_1+1:|\mathcal{R}_1|]\) and \(\sigma_2[s_2+1:|\mathcal{R}_2|]\).
    \item[Sequence \(r_3\)] The jobs \(\sigma_2[1:s_2]\), followed by \(\sigma_1[1:s_1]\), then \(\sigma_2[s_2+1:|\mathcal{R}_2|]\) and \(\sigma_1[s_1+1:|\mathcal{R}_1|]\).
\end{description}
Each sequence \(r_i\) is chosen with probability \(0\le\alpha_i\le 1\), such that \(\sum\alpha_i=1\). We are going to show how to set these \(\alpha_i\), depending on the structure of the confusion matrix, to minimize the randomized maximum regret \(\max_{\sigma\in S}\left\{\E\left[\tilde{C}^{\RH}_\sigma\right]-\tilde{C}^*_\sigma\right\}\).

First, remark that the sequence \(r_1\) can be decomposed as \(\sigma_1[1:s_1]\) followed by \(\sigma_1[s_1+1:|\mathcal{R}_1|]\), then \(\sigma_2[1:s_2]\) followed by \(\sigma_2[s_2+1:|\mathcal{R}_2|]\). This means that each sequence \(r_1,r_2,r_3\) consists in the 4 same sub-chains arranged in a different order, which allows us to characterize the structure of the scenario that maximizes the randomized regret. Indeed, for any scenario \(\sigma\), we have
\[
    \E\left[\tilde{C}^{\RH}_\sigma\right]-\tilde{C}^*_\sigma
    =\E\left[\tilde{C}^{\RH}_\sigma-\tilde{C}^*_\sigma\right]
    =\sum_{i=1}^3 \alpha_i(\tilde{C}^{r_i}_\sigma-\tilde{C}^*_\sigma),
\]
and by Lemma~\ref{lem.worst-case-min-max-regret}, we know that, for each \(r_i\), there necessarily exists a worst-case scenario \(\hat{\sigma}\) where \(\hat{\sigma}_1[1:s_1]\), \(\hat{\sigma}_1[s_1+1:|\mathcal{R}_1|]\), \(\hat{\sigma}_2[1:s_2]\) and \(\hat{\sigma}_2[s_2+1:|\mathcal{R}_2|]\) are all non-decreasing. Therefore, there is necessarily such a scenario that maximizes the randomized regret. To fully characterize it, we need to find how many large jobs belong to \(\hat{\sigma}_1[1:s_1]\) and \(\hat{\sigma}_2[1:s_2]\). Let \(x_1\) (resp.\ \(x_2\)) denote the number of large jobs in \(\hat{\sigma}_1[1:s_1]\) (resp.\ \(\hat{\sigma}_2[1:s_2]\)). We immediately get the regret of each sequence \(r_i\):
\[
    \begin{cases}
        \tilde{C}^{r_1}_{\hat{\sigma}}-\tilde{C}^*_{\hat{\sigma}}
        =(\lambda_L-\lambda_s)((s_2-x_2)(L_1-x_1)+x_2(L_1-x_1)+x_1(s_2-x_2)), \\
        \tilde{C}^{r_2}_{\hat{\sigma}}-\tilde{C}^*_{\hat{\sigma}}
        =(\lambda_L-\lambda_s)(x_1x_2+x_2(L_1-x_1)+x_1(s_2-x_2)), \\
        \tilde{C}^{r_3}_{\hat{\sigma}}-\tilde{C}^*_{\hat{\sigma}}
        =(\lambda_L-\lambda_s)(x_1x_2+x_2(s_1-x_1)+x_1(L_2-x_2)). \\
    \end{cases}
\]
Therefore, \(\sum_{i=1}^3 \alpha_i(\tilde{C}^{r_i}_{\hat{\sigma}}-\tilde{C}^*_{\hat{\sigma}})=(\lambda_L-\lambda_s)(\alpha_1s_2L_1+\alpha_2(s_2x_1+L_1x_2)+\alpha_3(s_1x_2+L_2x_1)-x_1x_2)=(\lambda_L-\lambda_s)f(x_1,x_2)\), where the second equality follows by noticing that \(\alpha_1=1-\alpha_2-\alpha_3\) and defining \(f(x_1,x_2)=-x_1x_2+x_1(\alpha_2s_2+\alpha_3L_2)+x_2(\alpha_2L_1+\alpha_3s_1)+s_2L_1(1-\alpha_2-\alpha_3)\). Define also \(m_1=\min(s_1,L_1)\) and \(m_2=\min(s_2,L_2)\). Clearly, the function \(f\) is bi-linear in \(x_1,x_2\), and as \(0\le x_1\le m_1\), \(0\le x_2\le m_2\), it is maximized either on \((0,0)\), \((m_1,0)\), \((0,m_2)\) or \((m_1,m_2)\). As \(f(0,0)\le f(m_1,0)\) and \(f(0,0)\le f(0,m_2)\), we can directly eliminate the case \((0,0)\).

Now we know that the expression \(\sum_{i=1}^3 \alpha_i(\tilde{C}^{r_i}_{\hat{\sigma}}-\tilde{C}^*_{\hat{\sigma}})\) is maximized in one of the 3 following cases: (i) \(x_1=m_1\) and \(x_2=0\), (ii) \(x_1=0\) and \(x_2=m_2\), or (iii) \(x_1=m_1\) and \(x_2=m_2\). We obtain
\[
    \sum_{i=1}^3 \alpha_i(\tilde{C}^{r_i}_{\hat{\sigma}}-\tilde{C}^*_{\hat{\sigma}})
    =(\lambda_L-\lambda_s)\max\left\{g_{(m_1,0)}(\alpha),g_{(0,m_2)}(\alpha),g_{(m_1,m_2)}(\alpha)\right\},
\]

where \[
    \begin{cases}
        g_{(m_1,0)}(\alpha)=\alpha_2(m_1s_2-s_2L_1)+\alpha_3(m_1L_2-s_2L_1)+s_2L_1, \\
        g_{(0,m_2)}(\alpha)=\alpha_2(m_2L_1-s_2L_1)+\alpha_3(m_2s_1-s_2L_1)+s_2L_1, \\
        g_{(m_1,m_2)}(\alpha)=g_{(m_1,0)}(\alpha)+g_{(0,m_2)}(\alpha)-m_1m_2-s_2L_1(1-\alpha_2-\alpha_3).
    \end{cases}
\]
The proof of Theorem~\ref{th.min-max-regret-rand} consists in solving the following min-max problem for each possible value of \(m_1,m_2\):
\begin{equation}
    \label{pb.min-max-alpha}
    \tag{\textsc{Min-Max-Alpha}}
    \min_{\alpha\in\Gamma}
    \left\{\max\left\{g_{(m_1,0)}(\alpha),g_{(0,m_2)}(\alpha),g_{(m_1,m_2)}(\alpha)\right\}\right\},
\end{equation}
where \(\Gamma\) is the simplex \(\{(\alpha_2,\alpha_3):\alpha_2+\alpha_3\le 1\}\). A detailed functional analysis (given in Appendix~\ref{sec.optimal-min-max-alpha}) shows that \textsc{Min-Max-Alpha} is solved by
\[
    \alpha=\begin{cases}
        \left(\frac{s_1(L_2-s_2)}{L_1L_2-s_1s_2},\frac{s_2(L_1-s_1)}{L_1L_2-s_1s_2}\right) & \text{if }s_1<L_1\text{ and }s_2\le L_2,\text{or }s_1\le L_1\text{ and }s_2<L_2, \\[4pt]
        \left(\frac{L_2(s_1-L_1)}{s_1s_2-L_1L_2},\frac{L_1(s_2-L_2)}{s_1s_2-L_1L_2}\right) & \text{if }L_1<s_1\text{ and }L_2\le s_2,\text{or }L_1\le s_1\text{ and }L_2<s_2, \\[4pt]
        (0,0) & \text{otherwise.} \\
    \end{cases}
\]

\subsection{Adaptive Algorithms for \textsc{Min-Expectation}} \label{subsec:adaptive}

We now consider \emph{adaptive} algorithms, which may observe the true processing time of each completed job and update their future scheduling decisions accordingly. We show that, for the \textsc{Min-Expectation} objective, the ratio between optimal non-adaptive and optimal adaptive costs is asymptotically close to $4/3$ on a specific family of instances, giving a lower bound on the performance improvement one may expect by exploiting adaptivity.

In what follows, the state of the remaining jobs to process after the execution of $t$ jobs is denoted by $E^t$, where $E^0$ is the initial confusion matrix and $E^n$ is the final confusion matrix without any remaining job. Each time a job drawn from queue $\mathcal{R}_k$ reveals true class $\ell$, one unit is subtracted from entry $(k,\ell)$, so the state transitions as $E^{t+1}=E^t-\mathbf{e}^{k\ell}$ where $\mathbf{e}^{k\ell}$ denotes the indicator matrix, i.e., the $K\times K$ matrix with a $1$ at position $(k,\ell)$ and $0$ elsewhere. Thus, $E^t$ records exactly the remaining unprocessed jobs after $t$ executions. In what follows, we extend the notations introduced in Section~\ref{sec:model} by adding $t$ as a superscript. Based on this, $e_{k\ell}^t$ is the number of remaining jobs at state $E^t$ that the classifier believes to belong to class $k$ but actually belong to class $\ell$, and $|\mathcal{R}_k^t|$ is the number of remaining jobs in queue $\mathcal{R}_k$ at state $E^t$. For any state $E^t$, we denote by $n_t = \sum_k |\mathcal{R}_k^t| = \sum_{k,\ell} e_{k\ell}^t$ the total number of remaining jobs.

\begin{remark}
Generalizing \textsc{Min-Cost} to the adaptive context, a natural greedy strategy consists in, at each step \(t\), always selecting the queue $\mathcal{R}^t_k$ with smallest current expected processing time $\bar{p}_k^t = \frac{1}{|\mathcal{R}_k^t|}\sum_\ell e_{k\ell}^t\lambda_\ell$, then observing the realized job, updating $E^{t+1}$ accordingly. The following $3\times 3$ instance shows this greedy strategy is not optimal. Set $\lambda_1=1$, $\lambda_2=3$, $\lambda_3=6$, and consider the matrix
\[
    E=\begin{pmatrix}
        0 & 0 & 0 \\
        0 & 1 & 0 \\
        1 & 0 & 1
    \end{pmatrix}
\]
Drawing from $\mathcal{R}_2$ yields a cost of
\[
    3\lambda_2+\frac{1}{2}\left((2\lambda_1+1\lambda_3)+(2\lambda_3+1\lambda_1)\right) = 19.5,
\]
whereas drawing from $\mathcal{R}_3$ yields a cost of
\[
    \frac{1}{2}\left(3\lambda_1+2\lambda_2+1\lambda_3\right)+\frac{1}{2}\left(3\lambda_3+2\lambda_1+1\lambda_2\right) = 19.
\]
The greedy rule selects first the queue $\mathcal{R}_2$ since $\bar{p}_2^0 = 3 < 3.5 = \bar{p}_3^0$, yielding an expected cost of 19.5, while optimal is 19, showing that greedy is strictly sub-optimal here.
\end{remark}

The optimal adaptive policy can be characterized via dynamic programming. Define $f$ by the Bellman recurrence:
\begin{align}
    \label{eq:bellman}
    f(E^t) & = \min_{\substack{1\le k\le K\\|\mathcal{R}_k^t|>0}}
    \Biggl\{\rule{0pt}{1.5ex} \sum_{\ell=1}^{K} \frac{e_{k\ell}^t}{|\mathcal{R}_k^t|} \cdot \Bigl[n_t\lambda_\ell+f\bigl(E^t-\mathbf{e}^{k\ell}\bigr)\Bigr]
    \Biggl\},
\end{align}
with the boundary condition $f(E^n)=0$, where $\frac{e_{k\ell}^t}{|\mathcal{R}_k^t|}$ is the probability that a job drawn from $\mathcal{R}_k$ belongs to class $\ell$,  $n_t\lambda_\ell$ is its immediate cost contribution to $\sum C_j$ (its processing time $\lambda_\ell$, and the delay $\lambda_\ell$ it inflicts on each of the $n_t\!-\!1$ remaining jobs), and $f\!\bigl(E^t\!-\!\mathbf{e}^{k\ell}\bigr)$ is the optimal future cost of the resulting state $E^{t+1}=E^t\!-\!\mathbf{e}^{k\ell}$. The proof of the following theorem is given in Appendix~\ref{appendix:proof_th_dp_opt}.

\begin{theorem} \label{thm:dp-optimal}
Let \textsc{DP} be a dynamic programming algorithm that, at each step $E^t$, selects the queue $\mathcal{R}_{k^*}$ that minimizes $f(E^t)$.
\textsc{DP} is an optimal adaptive algorithm of complexity $O\left(K\cdot e_{\max}^{K^2}\right)$, where $e_{\max}$ is the maximum entry of the initial confusion matrix $E=E^0$.
\end{theorem}

\begin{corollary}
If the number of classes $K$ is bounded by a constant, then \textsc{DP} is a pseudo-polynomial algorithm.
\end{corollary}

\begin{remark}
Note that $e_{\max} \leq n$, and then the complexity of \textsc{DP} can be written as $O\left(K\cdot n^{K^2}\right)$. However, this is not a polynomial complexity in the case of bounded $K$, since $n$ is an exponential parameter with respect to the size of the instance, which consists in the \(K\times K\) confusion matrix and the processing times of the classes $\lambda_1,\ldots,\lambda_K$.
\end{remark}

Next, our goal is to quantify how much the adaptive model can outperform the non-adaptive one.
All proofs of this part can be found in Appendix~\ref{appendix:price_non_adapt}.
We introduce the following family of instances.

\begin{definition}[Instances $\mathcal{M}_{K,x}$] \label{definition:familyM}
For any positive integer $K$ and any positive odd integer $x$, the instance $\mathcal{M}_{K,x}$ is composed by the $K\times K$ matrix where all entries of the first column have value $x$ and all entries of the last column have value 1 (and zeros elsewhere), along with the processing times $\lambda_1=0$ and $\lambda_K=1$.
\end{definition}

First, we compute the expected cost of the optimal non-adaptive algorithm \textsc{Min-Cost}.

\begin{lemma}\label{lem:exp-non-adapt}
For any instance $\mathcal{M}_{K,x}$, we have $\E[\tilde{C}^{\MC}(\mathcal{M}_{K,x})]=\frac{(x+1)}{2}K^2 +\frac{K}{2}$.
\end{lemma}

Since getting a closed formula for the expected cost of the optimal adaptive \textsc{DP} algorithm is hard (due to its recursive nature), the second ingredient of our analysis is the following non-optimal adaptive heuristic applied on any instance $\mathcal{M}_{K,x}$. By the optimality of \textsc{DP}, its expected cost is no bigger than the expected cost of this heuristic.

\begin{definition}[\textsc{Half-and-Check} (\HC) policy]\label{def:hcpolicy}
Let $\mathcal{L}$ be a list of jobs (initially empty), and let $\mathcal{M}_{K,x}$ be an instance as described in Definition~\ref{definition:familyM}. For each queue $\mathcal{R}_1,\mathcal{R}_2,\ldots,\mathcal{R}_K$ do:
\begin{enumerate}
   \item Execute the first $r=\frac{x+1}{2}$ jobs of current queue. We call these jobs the first \emph{half-block} of the queue.
   \item If the large job was observed in this first half-block, execute the remaining $r$ jobs of the current queue (second half-block) immediately, then proceed to the next queue.
   \item Otherwise, add the second half-block of current queue at the beginning of $\mathcal{L}$, and proceed to the next queue.
\end{enumerate}
Finally, execute the half-blocks stored in $\mathcal{L}$ respecting their order in the list.
\end{definition}

Next, we compute the expected cost of the adaptive algorithm \HC.

\begin{lemma}\label{lem:exp-adapt}
For any instance $\mathcal{M}_{K,x}$, we have
\[
    \E\left[\tilde{C}^{\HC}(\mathcal{M}_{K,x})\right] 
    =\begin{cases}
        \frac{15}{4}x + \frac{21}{4} & \text{if } K=3, \\
        \E\left[\tilde{C}^\HC(\mathcal{M}_{K-1,x})\right] + \frac{3}{4}(K-1)(x+1)+\frac{x}{2}+1 & \text{otherwise.}
    \end{cases}
\]
\end{lemma}

By combining Lemmas~\ref{lem:exp-non-adapt} and~\ref{lem:exp-adapt}, we get the following theorem.

\begin{theorem}\label{th:nonadapt-vs-adapt}
The expected-cost ratio between optimal non-adaptive and optimal adaptive solutions is asymptotically bigger than $\frac{4K}{3K+1}$.
\end{theorem}

\section{Conclusions and Perspectives}

We proposed a novel robustness model based on classification and confusion matrices, applied on the basic scheduling problem \(1||\sum C_j\). This allowed us to refine the analysis in non-clairvoyant setting: instead of being forced to rely on a random ordering of the jobs, we were able to leverage additional information provided by the classifier to derive new robustness guarantees, both in worst (\textsc{Min-Max Regret}) and average (\textsc{Min-Expectation}) cases. We think this model might constitute a new approach to analyze combinatorial optimization problems subject to uncertainty.

One logical next step consists in designing algorithms giving optimal results when the classifier provides perfect predictions (e.g., diagonal matrices), while being robust no matter the confusion matrix, in the same spirit as learning-augmented algorithms. In particular, we plan to study how the structure of the matrix impacts robustness. Another direction could be the generalization of the model to multi-label classification, where an object is categorized according to several attributes. For example, we could analyze the \(1||\sum w_jC_j\) problem, where each job is assigned one class for its weight, and a second class for its processing time.

\newpage

\printbibliography

\appendix

\newpage
\section{Optimal Algorithm for \(1|\text{chains}|\sum C_j\)}
\label{sec.chain-opt}

\textsc{Chain-Opt} (Algorithm~\ref{alg.chain-opt}) describes the optimal algorithm for \(1|\text{chains}|\sum C_j\)~\cite{conway1967theory}. Recall that each \(\mathcal{R}_k\) denotes the set of jobs believed to be in class \(k\), and that each \(\sigma_k\) represents the permutation of \(\mathcal{R}_k\). Thus, we can see each \(\sigma_k\) as a chain of jobs, and \(p_{\sigma_k[i]}\) gives the processing time of the \(i\)-th job in \(\mathcal{R}_k\).

At each step, the algorithm computes the average processing time of each prefix of each chain \(\sigma_k\) (\(\texttt{cost}_k(j)\) denotes the average processing time of the prefix of length \(j\) in \(\sigma_k\)), and schedules the jobs of the prefix of minimum cost. We assume that the algorithm always maximizes the number of jobs scheduled at each step, that is, if \(\texttt{best}_{k^*}=\texttt{cost}_{k^*}(j)=\texttt{cost}_{k^*}(j')\) for two distinct positions \(j<j'\), we systematically set \(\texttt{len}_{k^*}=j'\).

\begin{algorithm}[htbp]
    \caption{\textsc{Chain-Opt}}
    \label{alg.chain-opt}
    \begin{algorithmic}[1]
        \Require a confusion matrix \(E\) and a scenario \(\sigma\)
        \Ensure an optimal schedule
        \For{all queues \(\mathcal{R}_k\)}
            \If{\(\mathcal{R}_k=\emptyset\)}
                \State \(\texttt{best}_k\gets +\infty\)
                \State \(\texttt{len}_k\gets -1\)
            \Else
                \State \(\texttt{cost}_k(j)\gets \frac{1}{j}\sum_{i=1}^j p_{\sigma_k[i]}\) for all \(1\le j\le |\mathcal{R}_k|\)
                \State \(\texttt{best}_k\gets \min_{1\le j\le |\mathcal{R}_k|} \texttt{cost}_k(j)\)
                \State \(\texttt{len}_k\gets \argmin_{1\le j\le |\mathcal{R}_k|} \texttt{cost}_k(j)\)
            \EndIf
        \EndFor
        \While{there remain jobs to schedule}
            \State \(k^*\gets\argmin_{1\le k\le K} \texttt{best}_k\)
            \State Pull the first \(\texttt{len}_{k^*}\) jobs of \(\mathcal{R}_{k^*}\)
            \If{\(\mathcal{R}_{k^*}=\emptyset\)}
                \State \(\texttt{best}_{k^*}\gets +\infty\)
                \State \(\texttt{len}_{k^*}\gets -1\)
            \Else
                \State Update \(\texttt{cost}_{k^*}(j)\) for all \(1\le j\le |\mathcal{R}_{k^*}|\)
                \State Update \(\texttt{best}_{k^*}\) and \(\texttt{len}_{k^*}\)
            \EndIf
        \EndWhile
    \end{algorithmic}
\end{algorithm}

\newpage
\section{Optimal Solution of \ref{pb.min-max-alpha} (Theorem~\ref{th.min-max-regret-rand})}
\label{sec.optimal-min-max-alpha}

We recall the \ref{pb.min-max-alpha} problem:
\begin{equation}
    \tag{\textsc{Min-Max-Alpha}}
    \min_{\alpha\in\Gamma}
    \left\{\max\left\{g_{(m_1,0)}(\alpha),g_{(0,m_2)}(\alpha),g_{(m_1,m_2)}(\alpha)\right\}\right\},
\end{equation}
where \(\Gamma=\{(\alpha_2,\alpha_3):\alpha_2+\alpha_3\le 1\}\), \(m_k=\min(s_k,L_k)\), and
\[
    \begin{cases}
        g_{(m_1,0)}(\alpha)=\alpha_2(m_1s_2-s_2L_1)+\alpha_3(m_1L_2-s_2L_1)+s_2L_1, \\
        g_{(0,m_2)}(\alpha)=\alpha_2(m_2L_1-s_2L_1)+\alpha_3(m_2s_1-s_2L_1)+s_2L_1, \\
        g_{(m_1,m_2)}(\alpha)=g_{(m_1,0)}(\alpha)+g_{(0,m_2)}(\alpha)-m_1m_2-s_2L_1(1-\alpha_2-\alpha_3).
    \end{cases}
\]

First, suppose that \(s_1=L_1\) and \(s_2=L_2\). We immediately obtain \(g_{(m_1,0)}(\alpha)=g_{(0,m_2)}(\alpha)=s_2L_1\) and \(g_{(m_1,m_2)}(\alpha)=s_2L_1(\alpha_2+\alpha_3)\le s_2L_1\), thus \(\sum_{i=1}^3 \alpha_i(\tilde{C}^{r_i}_{\hat{\sigma}}-\tilde{C}^*_{\hat{\sigma}})=(\lambda_L-\lambda_s)s_2L_1\). We now assume \(s_1\ne L_1\) or \(s_2\ne L_2\), and we notice that the case \(m_1=s_1\) and \(m_2=L_2\) cannot happen, as we necessarily have \(s_2L_1\le s_1L_2\). This leaves 3 different cases to analyze.

\begin{description}
    \item[Case 1: \(m_1=s_1\) and \(m_2=s_2\).] We get
    \[
        \begin{cases}
            g_{(m_1,0)}(\alpha)=\alpha_2(s_1s_2-s_2L_1)+\alpha_3(s_1L_2-s_2L_1)+s_2L_1, \\
            g_{(0,m_2)}(\alpha)=\alpha_3(s_1s_2-s_2L_1)+s_2L_1, \\
            g_{(m_1,m_2)}(\alpha)=\alpha_2s_1s_2+\alpha_3(s_1L_2+s_1s_2-s_2L_1)+s_2L_1-s_1s_2.
        \end{cases}
    \]
    Observe that \ref{pb.min-max-alpha} can be rewritten as
    \[
        \min_{0\le\alpha_3\le 1}
        \left\{\min_{0\le\alpha_2\le 1-\alpha_3}
        \left\{\max\left\{g_{(m_1,0)}(\alpha_2,\alpha_3),g_{(0,m_2)}(\alpha_3),g_{(m_1,m_2)}(\alpha_2,\alpha_3)\right\}\right\}\right\}.
    \]
    First we solve the inner problem for any fixed \(\alpha_3\). We have that \(g_{(m_1,0)}\) decreases with \(\alpha_2\) (as \(s_1\le L_1\)), \(g_{(0,m_2)}\) is constant with \(\alpha_2\), and \(g_{(m_1,m_2)}\) increases with \(\alpha_2\). This implies that \(\max\left\{g_{(m_1,0)},g_{(m_1,m_2)}\right\}\) is minimized when \(g_{(m_1,0)}(\alpha_2,\alpha_3)=g_{(m_1,m_2)}(\alpha_2,\alpha_3)\), i.e., \(\alpha_2=(1-\alpha_3)\frac{s_1}{L_1}\). In other words,
    \[
        \min_{0\le\alpha_2\le 1-\alpha_3}
        \left\{\max\left\{g_{(m_1,0)}(\alpha_2,\alpha_3),g_{(0,m_2)}(\alpha_3),g_{(m_1,m_2)}(\alpha_2,\alpha_3)\right\}\right\}
        =\max\left\{h(\alpha_3),g_{(0,m_2)}(\alpha_3)\right\},
    \]
    where
    \begin{align*}
        h(\alpha_3)
        &=g_{(m_1,0)}\left((1-\alpha_3)\frac{s_1}{L_1},\alpha_3\right) \\
        &=\alpha_3\left(s_1L_2-s_2L_1-\frac{s_1}{L_1}(s_1s_2-s_2L_1)\right)+s_2L_1+\frac{s_1}{L_1}(s_1s_2-s_2L_1).
    \end{align*}
    The \ref{pb.min-max-alpha} problem becomes
    \[
        \min_{0\le\alpha_3\le 1}
        \left\{\max\left\{h(\alpha_3),g_{(0,m_2)}(\alpha_3)\right\}\right\}.
    \]

    We have \(s_2L_1\le s_1L_2\) and \(s_1\le L_1\), thus \(h\) increases with \(\alpha_3\) and \(g_{(0,m_2)}\) decreases with \(\alpha_3\). Again, this implies that \(\max\left\{h,g_{(0,m_2)}\right\}\) is minimized when \(h(\alpha_3)=g_{(0,m_2)}(\alpha_3)\), i.e.,
    \[
        \alpha_3=\frac{s_2(L_1-s_1)}{L_1L_2-s_1s_2}.
    \]

    We conclude that \(\sum_{i=1}^3 \alpha_i(\tilde{C}^{r_i}_{\hat{\sigma}}-\tilde{C}^*_{\hat{\sigma}})\) is minimized for \(\alpha_1=\frac{(L_1-s_1)(L_2-s_2)}{L_1L_2-s_1s_2}\), \(\alpha_2=\frac{s_1(L_2-s_2)}{L_1L_2-s_1s_2}\) and \(\alpha_3=\frac{s_2(L_1-s_1)}{L_1L_2-s_1s_2}\), and we obtain
    \[
        \E\left[\tilde{C}^{\RH}_{\hat{\sigma}}\right]-\tilde{C}^*_{\hat{\sigma}}
        =(\lambda_L-\lambda_s)\left(s_2L_1-\frac{s_2^2(L_1-s_1)^2}{L_1L_2-s_1s_2}\right).
    \]

    \item[Case 2: \(m_1=L_1\) and \(m_2=L_2\).] We get
    \[
        \begin{cases}
            g_{(m_1,0)}(\alpha)=\alpha_3(L_1L_2-s_2L_1)+s_2L_1, \\
            g_{(0,m_2)}(\alpha)=\alpha_2(L_1L_2-s_2L_1)+\alpha_3(s_1L_2-s_2L_1)+s_2L_1, \\
            g_{(m_1,m_2)}(\alpha)=\alpha_2L_1L_2+\alpha_3(s_1L_2+L_1L_2-s_2L_1)+s_2L_1-L_1L_2.
        \end{cases}
    \]
    With a reasoning analogous to Case 1, we deduce that \(\sum_{i=1}^3 \alpha_i(\tilde{C}^{r_i}_{\hat{\sigma}}-\tilde{C}^*_{\hat{\sigma}})\) is minimized for \(\alpha_1=\frac{(s_1-L_1)(s_2-L_2)}{s_1s_2-L_1L_2}\), \(\alpha_2=\frac{L_2(s_1-L_1)}{s_1s_2-L_1L_2}\) and \(\alpha_3=\frac{L_1(s_2-L_2)}{s_1s_2-L_1L_2}\), and we obtain
    \[
        \E\left[\tilde{C}^{\RH}_{\hat{\sigma}}\right]-\tilde{C}^*_{\hat{\sigma}}
        =(\lambda_L-\lambda_s)\left(s_2L_1-\frac{L_1^2(s_2-L_2)^2}{s_1s_2-L_1L_2}\right).
    \]

    \item[Case 3: \(m_1=L_1\) and \(m_2=s_2\).] We get
    \[
        \begin{cases}
            g_{(m_1,0)}(\alpha)=\alpha_3(L_1L_2-s_2L_1)+s_2L_1, \\
            g_{(0,m_2)}(\alpha)=\alpha_3(s_1s_2-s_2L_1)+s_2L_1, \\
            g_{(m_1,m_2)}(\alpha)=\alpha_2s_2L_1+\alpha_3(L_1L_2+s_1s_2-s_2L_1).
        \end{cases}
    \]
    We immediately see that \(g_{(m_1,m_2)}\) increases with \(\alpha_2\), and the 3 functions all increase with \(\alpha_3\). Thus, \(\sum_{i=1}^3 \alpha_i(\tilde{C}^{r_i}_{\hat{\sigma}}-\tilde{C}^*_{\hat{\sigma}})\) is minimized for \(\alpha_1=1\), \(\alpha_2=0\) and \(\alpha_3=0\), and we obtain
    \[
        \E\left[\tilde{C}^{\RH}_{\hat{\sigma}}\right]-\tilde{C}^*_{\hat{\sigma}}
        =(\lambda_L-\lambda_s)s_2L_1.
    \]
\end{description}

\newpage 

\section{Proof of Theorem~\ref{thm:dp-optimal}\label{appendix:proof_th_dp_opt}}

\noindent\textbf{Theorem~6.} \textit{Let \textsc{DP} be a dynamic programming algorithm that, at each step $E^t$, selects the queue $\mathcal{R}_{k^*}$ that minimizes $f(E^t)$.}

\textit{\textsc{DP} is an optimal adaptive algorithm of complexity $O\left(K\cdot e_{\max}^{K^2}\right)$, where $e_{\max}$ is the maximum entry of the initial confusion matrix $E=E^0$.}
\begin{proof}
For the optimality of \textsc{DP}, the proof proceeds by induction on $E^t$.

\medskip \noindent\textbf{Base case ($E^n$).} The base case $f(E^n)=0$ is immediate since the matrix $E^n$ contains no remaining jobs.

\medskip \noindent\textbf{Inductive step.}
Consider any adaptive algorithm $\mathcal{A}$, deterministic or randomized, and a state $E^t$.
By the inductive hypothesis, we know that for every state $E^i$ with $i>t$, the value $f(E^i)$ is a lower bound on the expected cost of $\mathcal{A}$.
It is sufficient to show that the total cost of $\mathcal{A}$ for $E^t$ is not smaller that $f(E^t)$.

Suppose that $\mathcal{A}$ selects the queue $\mathcal{R}_k$ at state $E^t$.
Conditioned on this choice, the revealed job is of class $\ell$ with probability $\frac{e_{k\ell}^t}{|\mathcal{R}_k^t|}$. 
This incurs an immediate cost $n_t\lambda_\ell$ and transitions to state $E^t- \mathbf{e}^{k\ell}$.
By the induction hypothesis applied to $E^t- \mathbf{e}^{k\ell}$, the expected future cost of $\mathcal{A}$ from that state is at least $f\bigl(E^t - \mathbf{e}^{k\ell}\bigr)$.
Hence the expected total cost of $\mathcal{A}$ is at least
$$\sum_{\ell=1}^K \frac{e_{k\ell}^t}{|\mathcal{R}_k^t|}\cdot\Bigl[n_t\lambda_\ell + f(E^t-\mathbf{e}^{k\ell})\Bigr] 
\geq \sum_{\ell=1}^K \frac{e_{k^*\ell}^t}{|\mathcal{R}_{k^*}^t|}\cdot\Bigl[n_t\lambda_\ell + f(E^t-\mathbf{e}^{k^*\ell})\Bigr] 
= f(E^t) $$
where the inequality follows by theorem's statement.

For the complexity of \textsc{DP}, consider an instance of $n$ jobs defined by a $K\times K$ matrix $E$.
For each entry $e_{k\ell}$ of $E$, the dynamic program needs to examine all admissible values $\{0,1,\cdots, e_{k\ell}\}$. Hence, the number of states generated by \textsc{DP} is at most 
$$\prod_{k,\ell} (e_{k\ell}+1) \leq \prod_{k,\ell} (e_{\max}+1) = \prod_{i=1}^{K^2} (e_{\max}+1) = (e_{\max}+1)^{K^2}$$
Taking into account that, for each state, we need to compute a minimum over $K$ values, the theorem follows.

\qed\end{proof}

\newpage 
\section{Expected-Cost Ratio Between Optimal Non-Adaptive and Adaptive Solutions} \label{appendix:price_non_adapt}

Our goal is to quantify how much the adaptive model can outperform the non-adaptive one.
We introduce the following family of instances.

\medskip\noindent\textbf{Definition 1 (Instances $\mathcal{M}_{K,x}$).} \textit{For any positive integer $K$ and any positive odd integer $x$, the instance $\mathcal{M}_{K,x}$ is composed by the $K\times K$ matrix where all entries of the first column have value $x$ and all entries of the last column have value 1, along with the processing times $\lambda_1=0$ and $\lambda_K=1$.}

\medskip

First, we compute the expected cost of the optimal non-adaptive algorithm \textsc{Min-Cost}.

\medskip\noindent\textbf{Lemma 4.} \textit{For any instance $\mathcal{M}_{K,x}$, we have $\E[\tilde{C}^{\MC}(\mathcal{M}_{K,x})]=\frac{(x+1)}{2}K^2 +\frac{K}{2}$.}
\begin{proof}
For any instance $\mathcal{M}_{K,x}$ and any queue $\mathcal{R}_k$ of the confusion matrix, we have $\bar{p}_k=\frac{1}{x+1}$. Thus, we can immediately compute the expected objective of \textsc{Min-Cost}:
\begin{align*}
    \E[\tilde{C}^{\MC}(\mathcal{M}_{K,x})]
        &=n\E\!\left[p_{[1]}\!\right]+(n\!-\!1)\E\!\left[p_{[2]}\!\right]+\cdots+\E\!\left[p_{[n]}\!\right] \\
        &=\frac{1}{x+1}\!\sum_{i=1}^{n}\!i \\
        &=\frac{1}{x+1}\cdot\frac{n(n\!+\!1)}{2} \\
        &=\frac{(x+1)}{2}K^2 +\frac{K}{2},
\end{align*}
where $\E[p_{[j]}]$ denotes the expected processing time of the job that is in position $j$. The last equality follows by noticing that there are $K(x+1)$ jobs in the instance.
\end{proof}

Since getting a closed formula for the expected cost of the optimal adaptive \textsc{DP} algorithm is hard (due to its recursive nature), the second ingredient of our analysis is the following non-optimal adaptive heuristic applied on any instance $\mathcal{M}_{K,x}$. By the optimality of \textsc{DP}, its expected cost is no bigger than the expected cost of this heuristic.

\medskip\noindent\textbf{Definition 2 (\textsc{Half-and-Check} (\HC) policy).} \textit{Let $\mathcal{L}$ be a list of jobs (initially empty), and let $\mathcal{M}_{K,x}$ be an instance as described in Definition~\ref{definition:familyM}. For each queue $\mathcal{R}_1,\mathcal{R}_2,\ldots,\mathcal{R}_K$ do:
\begin{enumerate}
   \item Execute the first $r=\frac{x+1}{2}$ jobs of current queue. We call these jobs the first \emph{half-block} of the queue.
   \item If the large job was observed in this first half-block, execute the remaining $r$ jobs of the current queue (second half-block) immediately, then proceed to the next queue.
   \item Otherwise, add the second half-block of current queue at the beginning of $\mathcal{L}$, and proceed to the next queue.
\end{enumerate}
Finally, execute the half-blocks stored in $\mathcal{L}$ respecting their order in the list.}

\medskip

To analyse \HC, it is convenient not to reason in terms of individual jobs, but in terms of the \emph{blocks} of jobs that \HC manipulates as atomic units, and to track, for each block, the cost it contributes internally, as well as the cost it imposes on the blocks scheduled after it.

\medskip\noindent\textbf{Block notation.} We introduce the following three atomic schedule units:
\begin{itemize}
  \item $\mathbf{b_s}$: a half-block of size $r = \frac{x+1}{2}$ containing
    only small jobs (with processing time 0). 
  \item $\mathbf{b_L}$: a half-block of size $r = \frac{x+1}{2}$ containing
    one large job (with processing time 1) and $r-1 = \frac{x-1}{2}$ small jobs.
    The large job occupies a uniformly random position within the block.
  \item $\mathcal{R}_k$: the full queue $\mathcal{R}_k$ (of size $x+1$), containing
    one large job and $x$ small jobs. The large job occupies a uniformly random
    position within the block.
\end{itemize}
We first quantify the cost each block type contributes to its own jobs, which we call its \emph{internal cost} $\delta$: this is the contribution to $\sum C_j$ coming only from delays that jobs of the block impose to the other jobs of the same block. Once a block's internal cost is known, the block can be treated as an opaque unit whose only remaining effect is to delay the jobs scheduled after it. This is precisely the decomposition used in the proof of Lemma~5 below.

\begin{lemma}[Internal costs]
\label{lem:internal} The internal costs of each block type are $\delta(\mathbf{b_s}) = 0$, $\delta(\mathbf{b_L}) = \frac{x+3}{4}$, and $\delta(\mathcal{R}_k) = \frac{x+2}{2}$.
\end{lemma}
\begin{proof}
Since $\lambda_1 = 0$, the completion time of any job equals the number of large jobs that precede it (including itself). The large job in $\mathbf{b_L}$ occupies position $j \in \{1, \ldots, r\}$ uniformly. It contributes $1$ to its own $C_j$, and delays the $r - j$ small jobs after it, each by $1$. Hence it contributes $r - j + 1$ to $\sum C_j$ in case $j$. Taking the expectation:
\[
    \delta(\mathbf{b_L}) 	= \frac{1}{r} \sum_{j=1}^{r} (r - j + 1)
  		= \frac{1}{r} \cdot \frac{r(r+1)}{2} = \frac{r+1}{2}
  		= \frac{\frac{x+1}{2} + 1}{2} = \frac{x+3}{4}.
 \]
The same calculation with $r' = x+1$ gives $\delta(\mathcal{R}_k) = \frac{x+2}{2}$.
\end{proof}

We now have all ingredients to prove Lemma~5. The key observation is that the terminal queue $\mathcal{R}_K$ is always executed as an uninterrupted block by \HC, regardless of where the large job falls within the earlier queues, i.e., for $\mathcal{R}_K$, either its $\mathbf{b_s}$ half-block is executed first and its $\mathbf{b_L}$ half-block is placed at the head of $\mathcal{L}$, or its $\mathbf{b_L}$ half-block is executed first and its $\mathbf{b_s}$ half-block follows immediately. We exploit this structural invariant to express $\E[\tilde{C}^\HC(\mathcal{M}_{K,x})]$ recursively, by rewriting the execution of a single terminal queue $\mathcal{R}_K$ as the execution of two consecutive queues $\mathcal{R}_K$ and $\mathcal{R}_{K+1}$, consistent with how \HC schedules them. This requires tracking, for each instance size, the expected number of jobs executed after $\mathcal{R}_K$ and the expected number of large jobs executed before it.

\medskip\noindent\textbf{Lemma 5.} \textit{For any instance $\mathcal{M}_{K,x}$, we have
\[
    \E\left[\tilde{C}^{\HC}(\mathcal{M}_{K,x})\right] 
    =\begin{cases}
        \frac{15}{4}x + \frac{21}{4} & \text{if } K=3, \\
        \E\left[\tilde{C}^\HC(\mathcal{M}_{K-1,x})\right] + \frac{3}{4}(K-1)(x+1)+\frac{x}{2}+1 & \text{otherwise.}
    \end{cases}
\]}
\begin{proof}
We prove the lemma by induction on $K$.

\medskip\noindent\textbf{Base case $\mathcal{M}_{3,x}$}. 
Each queue $\mathcal{R}_1$ and $\mathcal{R}_2$ independently has its large job in either the first or the second half-block, yielding four equally probable branches (each occurring with probability $1/4$):
\begin{center}
\renewcommand{\arraystretch}{1.3}
\small
\begin{tabular}{|l|l|l|l|}
\hline
$\mathcal{R}_1$ & $\mathcal{R}_2$ & Schedule order & Cost\\
\hline
Large in first half & Large in first half & $\mathbf{b_L},\mathbf{b_s},\mathbf{b_L},\mathbf{b_s},\mathcal{R}_3$ & $\delta(\mathbf{b_L})+5r+\delta(\mathbf{b_L})+3r+\delta(\mathcal{R}_3)$\\
Large in first half & Large in second half & $\mathbf{b_L},\mathbf{b_s},\mathbf{b_s},\mathcal{R}_3,\mathbf{b_L}$ & $\delta(\mathbf{b_L})+5r+\delta(\mathcal{R}_3)+1r+\delta(\mathbf{b_L})$\\
Large in second half & Large in first half & $\mathbf{b_s},\mathbf{b_L},\mathbf{b_s},\mathcal{R}_3,\mathbf{b_L}$ & $\delta(\mathbf{b_L})+4r+\delta(\mathcal{R}_3)+1r+\delta(\mathbf{b_L})$\\
Large in second half & Large in second half & $\mathbf{b_s},\mathbf{b_s},\mathcal{R}_3,\mathbf{b_L},\mathbf{b_L}$ & $\delta(\mathcal{R}_3)+2r+\delta(\mathbf{b_L})+1r+\delta(\mathbf{b_L})$\\
\hline
\end{tabular}
\end{center}
Averaging over the four branches and substituting $\delta(\mathbf{b_L})=\frac{x+3}{4}$,  $\delta(\mathcal{R}_3)=\frac{x+2}{2}$ and $r=\frac{x+1}{2}$ yields
\[
    \E\left[\tilde{C}^{\HC}(\mathcal{M}_{3,x})\right] = \frac{15}{4}x + \frac{21}{4}.
\]
We now establish the recurrence relating  $\mathcal{M}_{K+1}$ to  $\mathcal{M}_{K}$ for $K\geq3$.

\medskip\noindent\textbf{Step 1: Decomposition of \HC on $\mathcal{M}_{K,x}$.} The cost of \HC on  $\mathcal{M}_{K,x}$ can be written as
\begin{equation*}\label{eq:HCm_decomp}
    \E[\tilde{C}^\HC(\mathcal{M}_{K,x})] = \E[\tilde{C}^\HC(\mathcal{M}_{K,x})]^* + \delta(\mathcal{R}_K) + \E[A_K] + \E[B_K],
\end{equation*}
where $\E[\tilde{C}^\HC(\mathcal{M}_{K,x})]^*$ denotes the expected cost of all blocks \emph{excluding} $\mathcal{R}_K$ and its external effect:
\begin{itemize}
    \item $\delta(\mathcal{R}_K)$ is the internal cost of $\mathcal{R}_K$ (Lemma~\ref{lem:internal}),
    \item $\E[A_K]$ is the expected external cost imposed by $\mathcal{R}_K$ on the jobs that scheduled \emph{after} it (i.e., the expected number of jobs after $\mathcal{R}_K$, each delayed by exactly $\lambda_K$),
    \item $\E[B_K]$ is the expected external cost imposed on $\mathcal{R}_K$ by the jobs scheduled \emph{before} it (i.e., the expected number of large jobs executed before $\mathcal{R}_K$, each delaying $\mathcal{R}_K$ by exactly $\lambda_K$).
\end{itemize}
For each queue $\mathcal{R}_h$ with $h\in \{1,\ldots,K-1\}$, the large job is equally likely to fall in either half-block. If it falls in the second half-block (probability $1/2$), that half-block, containing $r$ jobs, is deferred after $\mathcal{R}_K$. Otherwise, nothing is deferred: since the $\mathbf{b_L}$ block is executed first, the $\mathbf{b_s}$ block of the same queue follows immediately. As queues are independent, the expected total number of deferred jobs is $\E[A_K]=(K-1)\frac{r}{2}$.

Every job executed after $\mathcal{R}_K$ belongs to one such deferred half-block, and each deferred half-block contains exactly one large job, hence $\frac{\E[A_K]}{r}$ of these jobs are large. The expected number of large jobs executed before $\mathcal{R}_K$ is therefore $K-\frac{\E[A_K]}{r}-1 = \frac{K-1}{2}$, and each of them delays the $(x+1)$ jobs of $\mathcal{R}_K$ by exactly $\lambda_K$, giving $\E[B_K]=(x+1)\frac{K-1}{2}$.

\medskip\noindent\textbf{Step 2: Substitution of $\mathcal{R}_K$ execution by $\mathcal{R}_K$ and $\mathcal{R}_{K+1}$ executions.} When a new queue $\mathcal{R}_{K+1}$ is appended to $\mathcal{M}_K$ to form $\mathcal{M}_{K+1}$, $\mathcal{R}_K$ is no longer executed as a single terminal block. Instead, $\mathcal{R}_K$ is processed by \HC like all other non-terminal queues: its first half-block is executed, and depending on whether the long job is found there, its second half-block either follows immediately (case $\mathbf{b_L}$) or is deferred (case $\mathbf{b_s}$). The new terminal block is $\mathcal{R}_{K+1}$, inserted at the position determined by the decision on queue $\mathcal{R}_K$.

Each existing branch of the $K$-queue tree (occurring with its original probability) now splits into two equally likely sub-branches:
\begin{itemize}
  \item \textbf{Case $\mathbf{b_L}$ (probability $1/2$):} the large job of $\mathcal{R}_K$ is in its first half-block. The schedule
  \[
    [\ldots], \mathcal{R}_K, [\text{\ldots}]
  \]
  within this branch becomes
  \[
    [\ldots], \mathbf{b_L}, \mathbf{b_s}, \mathcal{R}_{K+1}, [\text{\ldots}]
  \]
  The block $\mathcal{R}_{K+1}$ is followed by $A_K$ deferred jobs (same as before, since no new deferral is added).
  
  \item \textbf{Case $\mathbf{b_s}$ (probability $1/2$):} the large job of $\mathcal{R}_K$ is in its second half-block. The second half-block $\mathbf{b_L}$ is deferred. The schedule
  \[
    [\ldots], \mathcal{R}_{K}, [\text{\ldots}]
  \]
  becomes
  \[
    [\ldots], \mathbf{b_s}, \mathcal{R}_{K+1}, [\text{deferred } \mathbf{b_L}], [\dots]
  \]
  Hence, $\mathcal{R}_{K+1}$ is followed by $A_K + r$ deferred jobs.
\end{itemize}

\medskip\noindent\textbf{Step 3: Marginal cost of adding queue $\mathcal{R}_{K+1}$.} Both $\mathcal{R}_{K}$ and $\mathcal{R}_{K+1}$ are now delayed by all large tasks executed before the original execution of $\mathcal{R}_{K}$, which gives
\begin{align*}\label{eq:Gamma_raw}
    \E[\tilde{C}^\HC(\mathcal{M}_{K+1,x})]&-\E[\tilde{C}^\HC(\mathcal{M}_{K,x})]
    = \underbrace{-\delta(\mathcal{R}_K)-\E[A_K]-\E[B_K]}_{\text{remove }\mathcal{R}_K} \\
    & + \underbrace{\frac{1}{2}\left(\delta(b_L) + \left(3r+\E[A_K]\right) + \delta(b_S) + 0 + \delta(\mathcal{R}_{K+1}) + \E[A_K]+2\E[B_{K}]\right)}_{\text{case }\mathbf{b_L}\text{, proba. }1/2} \\
    & + \underbrace{\frac{1}{2}\left(\delta(b_S) + 0 + \delta(\mathcal{R}_{K+1}) + \left(r+\E[A_K]\right) + \delta(b_L) + \E[A_K]+2\E[B_{K}]\right)}_{\text{case }\mathbf{b_s}\text{, proba. }1/2}.
\end{align*}

\noindent Since $\delta(\mathcal{R}_K)=\delta(\mathcal{R}_{K+1})$, this simplifies to

\begin{align*}
\E[\tilde{C}^\HC(\mathcal{M}_{K+1,x})]-\E[\tilde{C}^\HC(\mathcal{M}_{K,x})] 
& = \delta(b_L)+2r+\E[A_K]+\E[B_K] \\
& = \frac{3}{4}K(x+1) + \frac{x}{2} + 1,
\end{align*}
which, after re-indexing $K+1 \to K$, is exactly the stated recurrence.
\end{proof}

By combining Lemmas~4 and~5, we get the following theorem.

\medskip\noindent\textbf{Theorem 7.} \textit{The expected-cost ratio between optimal non-adaptive and optimal adaptive solutions is asymptotically bigger than $\frac{4K}{3K+1}$.}
\begin{proof}
By solving the recurrence of Lemma~5, we get
\begin{equation*}
    \E\left[\tilde{C}^{\HC}(\mathcal{M}_{K,x})\right]
    = \frac{3}{4}\left(\frac{x+1}{2}\right)K^2 + \frac{x+5}{8}K.
\end{equation*}
Hence, for the expected-cost ratio of the two stated solutions we have
\begin{align*}
    \frac{\E\left[\tilde{C}^{\MC}(\mathcal{M}_{K,x})\right]}{\E\left[\tilde{C}^{\textsc{DP}}(\mathcal{M}_{K,x})\right]}
    &\geq\frac{\E\left[\tilde{C}^{\MC}(\mathcal{M}_{K,x})\right]}{\E\left[\tilde{C}^{\HC}(\mathcal{M}_{K,x})\right]}\\
    &=\frac{\frac{x+1}{2}K^2 + \frac{K}{2}}{\frac{3}{8}(x+1) K^2 + \frac{x+5}{8}K} \\
    &=\frac{4(x+1)K^2 + 4K}{3(x+1) K^2 + (x+5)K} \\
    &=\frac{4(x+1)K + 4}{3(x+1) K + x+5} \\
    &=\frac{4Kx + 4K + 4}{(3K+1)x + 3K + 5} \\
\end{align*}
By choosing a big enough value for $x$, the theorem follows.
\end{proof}

\end{document}